\documentclass[Afour,sageh,times]{sagej}

\usepackage{moreverb,url}
\usepackage[colorlinks,bookmarksopen,bookmarksnumbered,citecolor=red,urlcolor=red]{hyperref}
\usepackage[dvipsnames]{xcolor}
\usepackage{amsmath} 
\usepackage{amssymb} 
\usepackage{amsthm}
\usepackage{algpseudocode}
\usepackage{algorithm} 
\usepackage{algorithmicx}
\usepackage{stmaryrd}
\usepackage[font=small]{caption}
\usepackage{subcaption}
\usepackage{mathtools}

\newcommand\BibTeX{{\rmfamily B\kern-.05em \textsc{i\kern-.025em b}\kern-.08em
T\kern-.1667em\lower.7ex\hbox{E}\kern-.125emX}}

\newcommand{\Ac}{{\mathcal{A}}}

\newcommand{\Dc}{{\mathcal{D}}}

\newcommand{\Eb}{{\mathbb{E}}}

\newcommand{\Ibb}{{\mathbb{I}}}

\newtheorem{prop}{Proposition}
\newtheorem{assumption}{Assumption}

\newtheorem{definition}{Definition}

\newcommand\numberthis{\addtocounter{equation}{1}\tag{\theequation}}

\newcommand{\hX}{{\hat{X}}}
\newcommand{\hY}{{\hat{Y}}}
\newcommand{\hZ}{{\hat{Z}}}
\newcommand{\hF}{{\hat{F}}}

\algrenewcommand\algorithmicrequire{\textbf{Inputs:}}
\algrenewcommand\algorithmicensure{\textbf{Output:}}

\def\volumeyear{2023}

\begin{document}

\newcommand{\ssep}{\colon}
\newcommand{\unsafeBag}{B}
\newcommand{\unsafeDataLowScoreCount}{N}
\newcommand{\hrulealg}[0]{\vspace{1mm} \hrule \vspace{1mm}}

\runninghead{Luo, Zhao, et al.}

\title{Sample-Efficient Safety Assurances using Conformal Prediction}

\author{Rachel Luo\affilnum{1}, Shengjia Zhao\affilnum{1}, Jonathan Kuck\affilnum{2}, Boris Ivanovic\affilnum{1}, Silvio Savarese\affilnum{1}, \\ Edward Schmerling\affilnum{1}, and Marco Pavone\affilnum{1}}

\affiliation{\affilnum{1}Stanford University, Stanford, CA 94305, USA \\
\affilnum{2}Dexterity, Inc., Redwood City, CA 94063, USA}

\corrauth{Rachel Luo, Stanford University, Stanford, CA 94305, U.S.A.}
\email{rsluo@stanford.edu}

\begin{abstract}
When deploying machine learning models in high-stakes robotics applications, the ability to detect unsafe situations is crucial. Early warning systems can provide alerts when an unsafe situation is imminent (in the absence of corrective action). To reliably improve safety, these warning systems should have a \textit{provable} false negative rate; i.e.\ of the situations that are unsafe, fewer than $\epsilon$ will occur without an alert. In this work, we present a framework that combines a statistical inference technique known as conformal prediction with a simulator of robot/environment dynamics, in order to tune warning systems to provably achieve an $\epsilon$ false negative rate using as few as $1/\epsilon$ data points. We apply our framework to a driver warning system and a robotic grasping application, and empirically demonstrate the guaranteed false negative rate while also observing a low false detection (positive) rate.
\end{abstract}

\keywords{Safety assurance, Conformal prediction, Statistical inference}

\maketitle

\section{Introduction}
Monitoring a system for faults, or detecting if unsafe situations will occur is a key problem for high-stakes robotics applications, and indeed the field of fault detection has long been the state of practice for building reliable systems~\citep{visinsky1994expert,visinsky1994robotic,visinsky1995dynamic,vemuri1998neural,khalastchi2018fault,muradore2011pls,crestani2015enhancing,Ding2013fault,Patton1997Observer,Harichi2015Model,Harirchi2017Guaranteed}. With the advent of learning-enabled components in robotic systems, robots are performing increasingly complex safety-critical tasks, so reliability has become increasingly important. For instance, in an autonomous driving setting, errors in perception or planning could lead to collision. In a warehouse robotics setting, robots on the factory floor work alongside humans, and not recognizing faults in learned systems could impact safety or even lead to injuries. At the same time, it is less clear how to ensure reliability for these learned systems. These systems are complex, so guaranteeing safety is not something that can be done from first principles --- empirical, data-driven methods are needed.

In this work, we present a sample efficient and principled method for detecting unsafe situations based on the statistical inference technique of conformal prediction~\citep{Vovk2005Algorithmic}. Our method provides \textit{provable} false negative rates for warning systems (i.e.\ among the situations in which an alert should be issued, fewer than $\epsilon$ occur without an alert), while achieving low false positive rates (few unnecessary alerts are issued). 

For example, in a driver assistance system, when an unsafe situation (e.g. another car getting too close) is imminent, our method will issue a warning the vast majority of the time (i.e. at least $1 - \epsilon$ of the time). In a warehouse setting with a robotic pick-and-place system, when the system will fail to grasp and transport an object, our method will issue an alert at least $1 - \epsilon$ of the time.
As a running example in this paper, we use our method to design an alert system to warn a human operator of impending danger in a driving application (illustrated in Figure~\ref{fig:system}). 

\begin{figure*}
    \centering
    \includegraphics[width=0.8\linewidth]{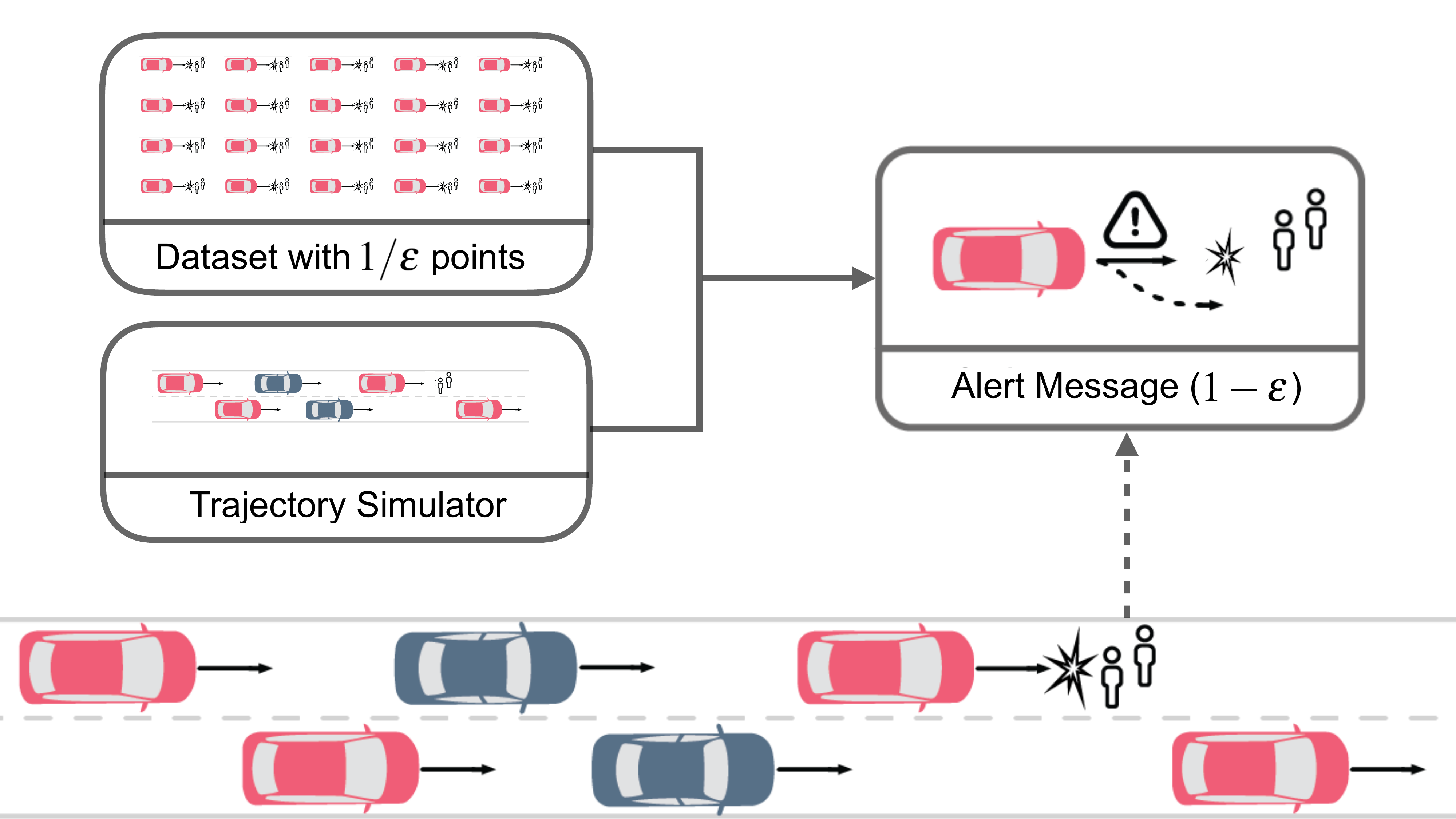}
    \caption{We design a warning system that achieves a \textit{provable} false negative rate sample efficiently. Among the situations that are dangerous (i.e.\ lead to an unsafe future situation in the absence of corrective action), fewer than $\epsilon$ occur without an alert.}
    \label{fig:system}
\end{figure*}

\subsection{Related Work}
Traditional fault detection techniques include hardware redundancy, signal processing, and plausibility tests~\citep{Ding2013fault,visinsky1994expert,visinsky1994robotic,visinsky1995dynamic,vemuri1998neural}. However, 
hardware redundancy requires extra components, signal processing works well only for processes in steady state, and plausibility tests do not catch faults that lead to a physically plausible system. 
Additionally, these methods typically lack performance guarantees. Model-based fault detection techniques~\citep{Ding2013fault,Patton1997Observer,Harichi2015Model,Harirchi2017Guaranteed} involve using a model of the system to determine whether a fault has occurred; they assume that users have a very accurate model of the system dynamics, which is difficult to obtain in practice. 

Another common approach for detecting unsafe states employs supervised learning to train a classifier model for labeling states as unsafe, and then the classifier hyperparameters are adjusted until empirically the false negative rate is low. In practice this is typically accomplished by plotting a receiver operating characteristic (ROC) curve and tuning the classification threshold to achieve low false negative rate. However, this approach requires training a new classification model, and provides no performance guarantees.

To guarantee the false negative rate of a learned warning system, the standard statistical learning framework could be used under standard i.i.d. assumptions~\citep{Luxburg2011Statistical}. A practitioner could collect additional data and use a validation dataset to provably certify the false negative rate. However, the key problem is \textit{data efficiency}, because collecting data for unsafe situations can be very expensive~\citep{Foody2009Sample,Luxburg2011Statistical,Calfiore2006Scenario}. 

\subsection{Contributions}

\paragraph{Main Question} Can we tune a warning system and guarantee a low false negative rate with only a handful of data points? 
For example, with only 30 data samples of dangerous situations, can we tune a warning system to have a \textit{provable} 5\% false negative rate? This problem is easy if we allow trivial systems that always issue a warning, but such systems are not practically useful. 
If we restrict our attention to non-trivial systems, 
this problem is seemingly impossible because even if a \textit{fixed} warning system successfully identifies all 30 dangerous situations, due to statistical fluctuations, we cannot prove that its false negative rate is less than 5\% (with high confidence). If proving that a fixed predictor achieves safety is difficult, tuning a predictor to provably achieve safety seems only more challenging. 

\paragraph{Our Contribution} We answer our main question affirmatively. We adapt a statistical inference framework known as conformal prediction to a robotics setting in order to tune systems to achieve provable safety guarantees (e.g. 5\% false negative rate) with extremely limited data (e.g. 30 samples). We only require a single assumption: the training samples are exchangeable with each test sample, i.e. for each test sample, if we permute the concatenated sequence of the training samples and the test sample, there is no reason to believe that any permutation is more or less likely to occur. This is a weaker assumption than the i.i.d. assumption typically used in the statistical learning framework. 

The assumption is reasonable in practice, even in situations in which the statistical inference assumptions fail (e.g.\ situations with temporal correlations between different test samples). In a driving scenario, for instance, the training dataset could include scene snippets sampled from different scenes, and these will be i.i.d. At test time, there may be one scene with several snippets. These snippets are obviously not independent; however, they are individually exchangeable with the training dataset. 

While this answer seems too good to be true, the key insight here is that we provide a type of guarantee that is different from  standard statistical learning guarantees.
Consider a sequence of test samples $Z_1, \cdots, Z_N$, and event indicators $F_1, \cdots, F_N$ for whether our warning system fails on each test sample. 
\begin{itemize}
    \item \textbf{Statistical learning guarantee.} In the statistical learning framework, we assume that the test samples $Z_1, \cdots, Z_N$ are i.i.d., so the failure events $F_1, \cdots, F_N$ are also i.i.d. --- we guarantee the failure probability for a sequence of i.i.d. failure events.
    \item \textbf{Conformal prediction guarantee.} In the conformal prediction framework, the test samples are not necessarily i.i.d., so the failure events $F_1, \cdots, \allowbreak F_N$ can be correlated  ---  we guarantee the \textit{marginal} failure probability for each failure event. In other words, we know that each test sample has a low probability of failure (i.e. $F_n = 1$ with low probability), but the failures could be correlated. For example, conditioning on $F_n = 1$ might increase or decrease the probability that $F_{n+1} = 1$ (while in the i.i.d. case, $F_n$ and $F_{n+1}$ are independent events).
\end{itemize}

The usefulness of the conformal guarantee depends on the intended application. Consider the driver alert system example: for individual drivers, collisions are rare and most drivers will not encounter more than one. Hence, there is little reason to worry about whether the warning failures are correlated between collisions. In other words, the conformal guarantee can convey confidence to individual users who rarely encounter multiple failures. 

On the other hand, the conformal guarantee may convey less confidence to a company with a large fleet of vehicles. For example, if $F_n = 1$ increases the probability that $F_{n+1} = 1$, then it is possible to have multiple simultaneous failures. However, this is not a limitation of our method, but rather an unavoidable consequence of the weaker (not i.i.d.) assumptions: if the test data is correlated (which we have no control over), then failure events of a warning system are inherently correlated. The weaker assumption is usually necessary because most robotics applications are deployed in time series or sequential decision making setups, so data from nearby time steps are correlated and not i.i.d. 
Since standard statistical learning guarantees are not applicable due to violation of the i.i.d. assumption, having some (conformal) guarantee is better than none. 

Furthermore, we will show empirically in Section~\ref{sec:experiments-driving} that failures are not highly correlated on two real-world driving datasets. Therefore, despite the lack of formal guarantees, there is strong empirical evidence suggesting that simultaneous failures do not occur in practice.

Thus, our contribution is four-fold: 
\begin{enumerate}
    \item We introduce a new notion of safety guarantee that is satisfactory for many use cases and has extremely good sample efficiency.
    \item We show how to leverage the statistical inference tool of conformal prediction for robotics applications.
    \item We instantiate a framework for applying conformal prediction to robotic safety. 
    \item We validate our framework experimentally on both a driver alert safety system and a robotic grasping system, showing that the conformal guarantees hold in practice, without issuing too many false positive alerts (e.g. less than 1\% for many setups). 
\end{enumerate}

A preliminary version of this article was presented at the 2022 Workshop on the Algorithmic Foundations of Robotics~\citep{this-paper}. In this revised and extended version, we additionally contribute: (1) additional exposition on the surrogate safety score in our proposed framework, (2) proofs for the propositions in Section~\ref{sec:framework}, (3) additional exposition on the details of our experimental setups, (4) experimental results demonstrating the tradeoff between $\epsilon$ and the false positive rate (FPR) when there are few samples, (5) experimental results demonstrating empirically that failures are not highly correlated in a real-world driving setting, and (6) experimental results demonstrating that worse surrogate safety scores lead to a performance drop in terms of the false positive rate (but no performance drop in terms of the false negative rate).

\subsection{Organization}
The rest of this paper is organized as follows. In Section~\ref{sec:background}, we review conformal prediction. In Section~\ref{sec:framework}, we describe our problem setup, introduce our framework, and demonstrate that specific choices for elements of our framework lead to instantiations such as tuning an ROC curve threshold to limit false negatives (though we enrich this classic method with new guarantees). We then explain the differences between the conformal prediction guarantees and the statistical learning guarantees, and discuss when our guarantees should be applied. Finally, in Sections~\ref{sec:experiments-driving} and~\ref{sec:experiments-grasp}, we evaluate our framework on a driver alert safety system and on a robotic grasping system.

\section{Overview of Conformal Prediction}
\label{sec:background}

This section provides an overview of conformal prediction, the general framework that we adapt for robotics safety. It may be skipped without breaking the flow of the paper. 

Consider a prediction problem where the input feature is denoted by $X$ and the label is denoted by $Z$. Conformal prediction~\citep{Shafer2008Tutorial} is a class of methods that can produce prediction sets (i.e. a set of labels), such that the true label belongs to the predicted set with high probability. 
In its standard form, 
conformal prediction requires two components: a sequence of validation data $(X_1, Z_1), \cdots, (X_T, Z_T)$ and a non-conformity score $\psi$, which is any function from the input feature $X$ and the label $Z$ to a real number. Intuitively, the non-conformity score should measure the ``unusualness'' of the label $Z$ when the input feature is $X$. An example non-conformity score is $\psi(X, Z) = |h(X) - Z|$ where $h$ is some fixed prediction function --- intuitively, $Z$ is ``unusual'' if the prediction function has large error. 

The conformal prediction algorithm computes the non-conformity score for all samples in the validation set. Given a new test example with input feature $\hX$, the conformal prediction algorithm then ``tries'' all possible labels $z$, and measures the non-conformity score $\psi(\hX, z)$. A label is rejected if the computed non-conformity score is greater than $1 - \epsilon$ of the non-conformity scores in the validation set. Any label that is not rejected is included in the prediction set. Intuitively, the true label is unlikely to have a non-conformity score higher than $1 - \epsilon$ of validation samples; hence the true label is unlikely to get rejected. 

If the training data and the new test data point $(\hX, \hZ)$ are exchangeable, i.e. the probability of observing any permutation of $(X_1, Z_1), \cdots, (X_T, Z_T), (\hX, \hZ)$ is equally likely, then conformal prediction has very strong validity guarantees: the true label will be within the 
prediction set with $1 - \epsilon \pm 1/(T+1)$ probability. We note that this guarantee holds regardless of the nonconformity function $\psi$.

There are many extensions of conformal prediction, and the most relevant extension to our safety application is 
Mondrian conformal prediction~\citep{Vovk2003Mondrian,Vovk2005Algorithmic}, which partitions the input data into several categories such that each data point belongs to exactly one category, and guarantees validity separately for each category. Our work is based on Mondrian conformal prediction; because we wish to limit the false negative rate in warning systems, we need class-conditional validity for samples in the ``unsafe'' class.

Works that apply conformal prediction to classification models in learned systems include~\citet{Angelopoulos2020UncertaintySF},~\citet{Ghosh2023ImprovingUQ}, and~\citet{angelopoulos2022recommendation}. ~\citet{Angelopoulos2020UncertaintySF} and~\citet{Ghosh2023ImprovingUQ} apply conformal prediction to image classifiers to obtain a predictive set containing the true label with a user-specified probability. ~\citet{angelopoulos2022recommendation} studies conformal prediction in the context of recommender systems, and uses it to produce a set of user recommendations.

Works that apply conformal prediction to robotics settings include~\citet{Chen2020ReactiveMP},~\citet{Cai2020RealtimeOD},~\citet{Nouretdinov2011MRI}, and~\citet{Gammerman2008ClinicalMS}. ~\citet{Chen2020ReactiveMP} uses conformal prediction to predict a set of possible future motion trajectories from out of a set of 17 basis trajectories; ~\citet{Cai2020RealtimeOD} uses some ideas from conformal prediction for detecting out of distribution samples in cyber-physical systems; and~\citet{Nouretdinov2011MRI} and~\citet{Gammerman2008ClinicalMS} use conformal prediction for medical diagnosis. However, these works consider very different targeted problems, while we consider the problem of warning systems and provide a general framework for using conformal prediction on a variety of robotics applications.  

\section{Conformal Prediction Framework for Robotics Applications}
\label{sec:framework}

\subsection{Problem Setup} 
\label{sec:setup}

We consider a model-based planning application where we have some existing simulator or model, and given the current observations (denoted by random variable $X$), the simulator or model predicts the future states of the system (denoted by $Y$) in the absence of a warning. For instance, many applications have off-the-shelf simulators: an autonomous driving software might simulate the future trajectories of all traffic participants (up to some time horizon), or an aircraft control software might forward simulate the dynamics of the aircraft. We will use the random variable $Z$ to denote the true unknown future states of the system in the absence of a warning, e.g. the true future trajectories of traffic participants, or the true future dynamics of an aircraft. 

In our setup, depending on the model or simulator available, $Y$ could have the same type as $Z$ (e.g. both $Y$ and $Z$ are random variables that represent the future trajectories of traffic participants), or $Y$ could have a different type from $Z$ (e.g.\ $Y$ might represent some but not all aspects about the future, such as the direction of movement for traffic participants, or the distance from collision). 

Figure~\ref{fig:setup} shows a simplified illustration of our running driver alert system example. In this scene, there is an ego-agent (shown in blue), and an external agent (shown in red), whose position we would like to predict. Here, $X$ represents the current locations of the agents in the scene (i.e. the location of the red car at the bottom of the figure), $Y$ represents the predicted future locations (in this illustration our model predicts that the red car will move up and to the left), and $Z$ represents the true unknown future locations (perhaps the red car actually moves up and to the right).

\begin{figure}
    \centering
    \includegraphics[width=0.75\linewidth]{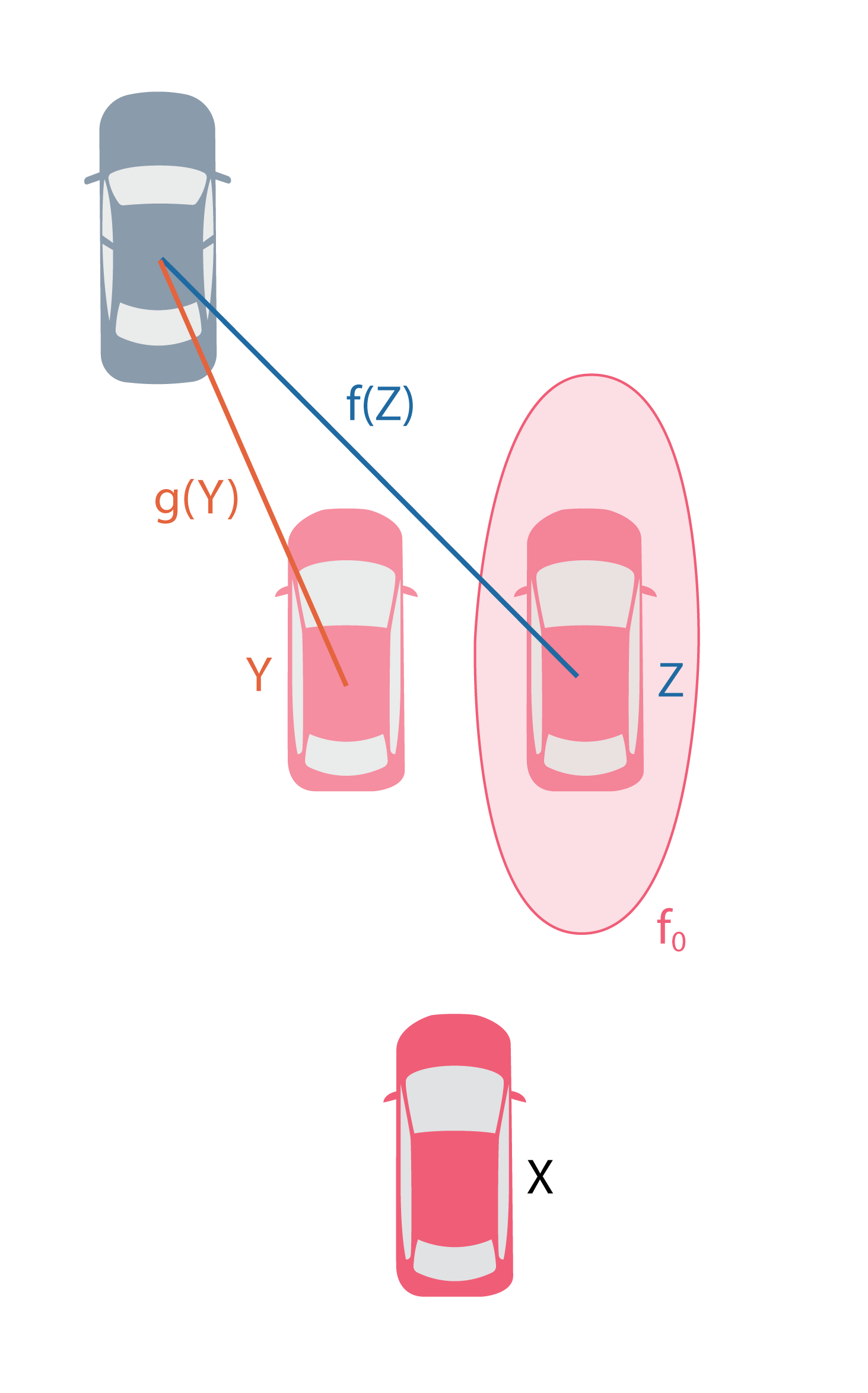}
    \caption{Overview of our problem setup: in this simplified example, there is an ego-agent (shown in {\color{Blue} blue}), and an external agent (shown in {\color{Red} red}), whose position we would like to predict. $X$ represents the current location of the external agent, $Y$ represents the predicted future location of the external agent, and $Z$ represents the true future location of the external agent. $f_0$ represents the distance threshold at which two cars are considered too close together (i.e. the situation is considered unsafe), and $f$ and $g$ are both the distance to the nearest car.}
    \label{fig:setup}
\end{figure}

\subsubsection{Assessing Safety} We assume that if we know the \textit{true} future state of the system $Z$, we can assess whether it is safe or not. Specifically, there exists some \textbf{safety score} denoted by $f(Z)$; we specify some threshold (denoted by $f_0$), and wish to be alerted if the safety score drops below this threshold (i.e.\ if $f(Z) < f_0$). In other words, a situation is defined to be unsafe if the safety score $f(Z)$ is too low. Most applications have natural safety scores. For instance, an autonomous driving safety score $f$ could be the distance to or time from collision; an aircraft control safety score could be the (negative) absolute difference between the orientation of the aircraft and its ideal orientation. 

In addition, we assume that the user provides a \textbf{surrogate safety score} $g: Y \mapsto \mathbb{R}$ that maps from the simulator prediction to a ``safety score,'' where a higher score indicates ``safe'' and a lower score indicates ``unsafe''. Ideally the surrogate safety score $g(Y)$ should be highly correlated with the true safety score $f(Z)$, but technically $g$ can be any function. None of our technical results depend on any assumptions about $g$; however, the choice of $g$ affects the empirical performance in terms of false positive rate (i.e. how often our warning system issues unnecessary alerts). When $Y$ and $Z$ have the same type, we can simply choose $g := f$; when $Y$ and $Z$ have different types we need to choose $g$ on a case-by-case basis. 

For example, in the driver assistance system in Figure~\ref{fig:setup}, the safety score $f(Z)$ could be the distance to the nearest car (shown by the blue line in the figure). Since the simulator $Y$ can output predicted distances to other cars, the surrogate safety score $g(Y)$ could also be the distance to the nearest car (shown by the orange line in the figure). The difference between $f$ and $g$ is that the input to $g$ is the \textit{predicted} future state of the system $Y$ (output by the simulator), rather than the \textit{true unknown} future state of the system $Z$. In this example, $f_0$ is the safety threshold shown by the red boundary; so if the red and blue cars are too close together, then the situation is considered unsafe.

\begin{figure*}[h]
    \centering
    \includegraphics[width=0.75\linewidth]{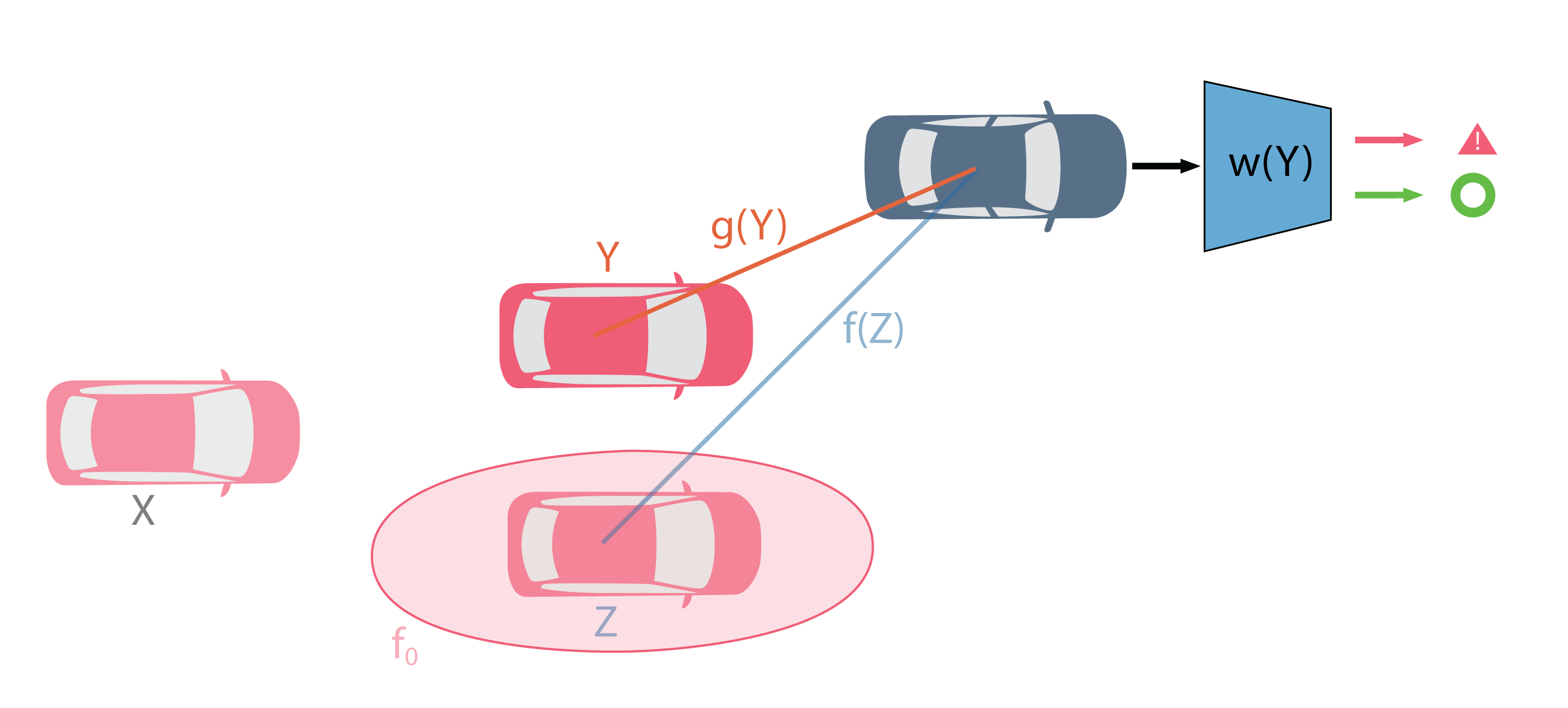}
    \caption{Given the simulation or model output $Y$, the warning function $w(Y)$ either decides to issue a warning ($w(Y) = 1$), or not ($w(Y) = 0$).}
    \label{fig:warning_system}
\end{figure*}

\begin{figure*}
    \centering
    \includegraphics[width=0.75\linewidth]{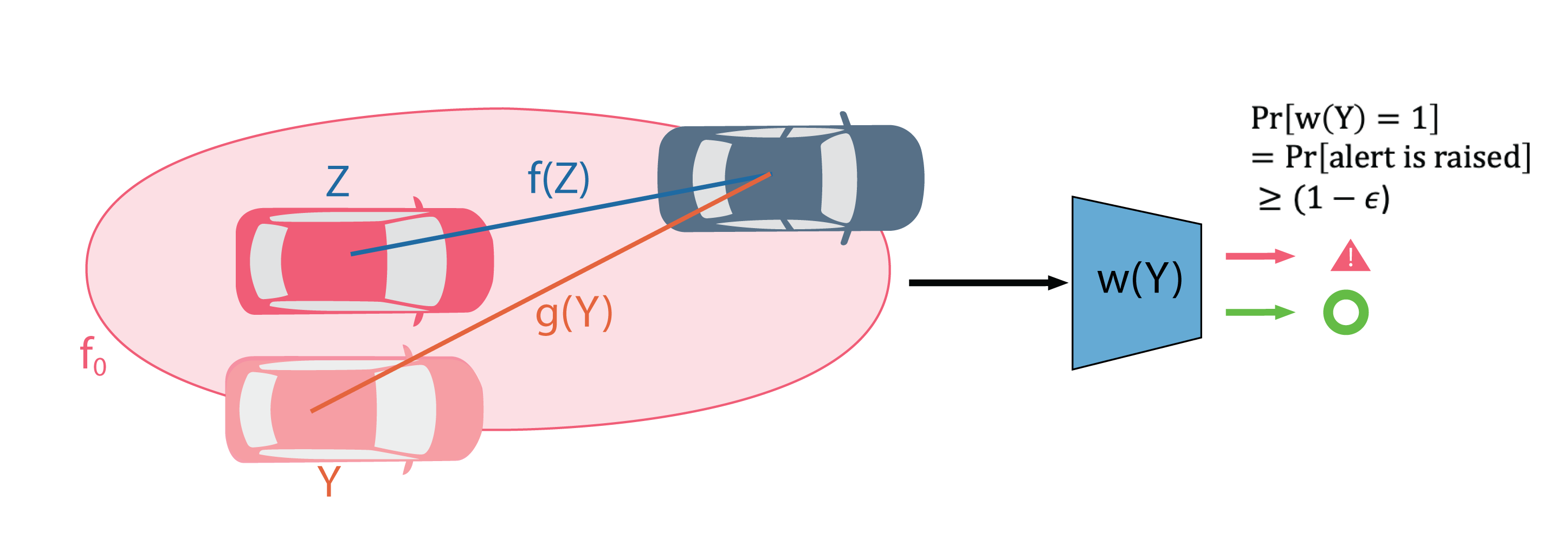}
    \caption{Whenever the true safety score $f(Z)$ is below $f_0$ (i.e. the red car and the blue car will indeed be too close together), the warning system should issue a warning ($w(Y) = 1$) with at least $1 - \epsilon$ probability.}
    \label{fig:epsilon_safety}
\end{figure*}

\subsubsection{Warning Function} We wish to design a warning function (denoted as a function $w(Y)$) that given the simulation or model output $Y$, decides to issue a warning ($w(Y) = 1$) or not ($w(Y) = 0$) (see Figure~\ref{fig:warning_system}). Note that $Y$ depends on previous states, so the warning function implicitly depends on previous observations through $Y$.
Formally we define ``safety'' as the following requirement:
\begin{definition}
\label{def:safe} 
For some $0 < \epsilon < 1$, we say that the warning system $w$ is $\epsilon$-safe (with respect to $Y,Z$, $f$, and $f_0$) if  
\begin{align*}
    \Pr[w(Y) = 1 \mid f(Z) < f_0] \geq 1 - \epsilon.
\end{align*}
\end{definition}
In words, whenever the true future safety score $f(Z)$ is below $f_0$, the warning system should issue a warning ($w(Y) = 1$) with at least $1-\epsilon$ probability (see Figure~\ref{fig:epsilon_safety}). Another way to think of this is that the false negative rate is at most $\epsilon$. The main difficulty here is that the warning function $w$ can depend on only the simulated future $Y$ rather than the true future $Z$ (which is not yet observed when the warning is issued), and the simulation might not come with any performance guarantees. 

A trivial warning system that always issues a warning (i.e. $w_{\mathrm{trivial}}(Y) \equiv 1$) is always $\epsilon$-safe for any $\epsilon > 0$. However, such a warning system is not useful. A useful warning system should issue as few warnings as possible when safe. Therefore, we should also consider its false positive rate
\begin{align*}
    \mathrm{FPR}(w) = \Pr[w(Y) = 1 \mid f(Z) \geq f_0].
\end{align*}
The false positive rate is of lower priority for safety because issuing an unnecessary warning might only be an inconvenience, while failing to issue a warning when the situation is unsafe can lead to catastrophic outcomes. In summary, our goal is to design a warning function $w(\cdot)$ such that:

\textbf{Goal:} Provably achieve $\epsilon$-safety for small $\epsilon$ (e.g. $0.02$), while achieving low  false positive rate (FPR). 

\subsubsection{Examples} A few examples that illustrate this problem setup are as follows: 
\begin{enumerate}
    \item In a driver alert system, users may want an assurance that among the instances in which the driver is in a dangerous situation, the system will issue a warning the vast majority of the time. The safety score in this case could be the time to collision (TTC), or the nearest distance from another car. 
    \item In a multi-arm robot collaboration system, users may want an assurance that among the instances in which the robot arms may collide, the system will issue a warning the majority of the time. The safety score could be the nearest distance to another robot arm. 
    \item In a warehouse robotic box-stacking system, users may want an assurance that among the instances in which the boxes will topple, the system will issue a warning the majority of the time. The safety score could be the probability of a stable stack. 
    \item In a coffee shop with a robotic barista, users may want an assurance that among the instances in which the robot may spill hot coffee, the system will issue a warning the majority of the time. The safety score could be the probability of a successful pour.
\end{enumerate}

Examples 3 and 4 can be thought of as ROC curve threshold tuning. If the model used is a binary classifier that predicts whether there is a stable stack, we can use the predicted probability of a ``safe'' outcome 
as $g$. Note that in this special case, our method also tunes the threshold, but adds guarantees on the false negative rate and practical guidelines for sample complexity. 

\subsection{Analysis of the Trade-off Between the FNR and FPR} 
\label{sec:fpr}

In this section, we analyze the fundamental trade-off between the false negative rate (FNR) and the false positive rate (FPR) of a warning system. For example, a trivial system that always issues a warning will have a 0\% FNR but 100\% FPR. Conversely, a system that never issues a warning will have a 0\% FPR but 100\% FNR. This suggests a trade-off between the achievable FNR and FPR. 

\subsubsection{Infinite Validation Data Regime} Even with infinite validation data, we may not be able to achieve both perfect FNR and perfect FPR because of inherent limitations of the safety score. For instance, at one extreme, if the safe and unsafe examples have identical safety score distributions, then there is no way to distinguish them according to Definition~\ref{def:safe}. At the other extreme, if the safe and unsafe examples have disjoint safety score distributions, then we can distinguish them perfectly (i.e. achieve 0\% FPR and 0\% FNR). A typical real world scenario will likely fall somewhere in between the two extremes, as illustrated in Figure~\ref{fig:optimal_fpr}. The key quantity is the amount of overlap between the safety score distribution for safe vs. unsafe examples, which will dictate the optimal achievable trade-off between the FPR and the FNR. 

\begin{figure}
    \centering
    \includegraphics[width=\linewidth]{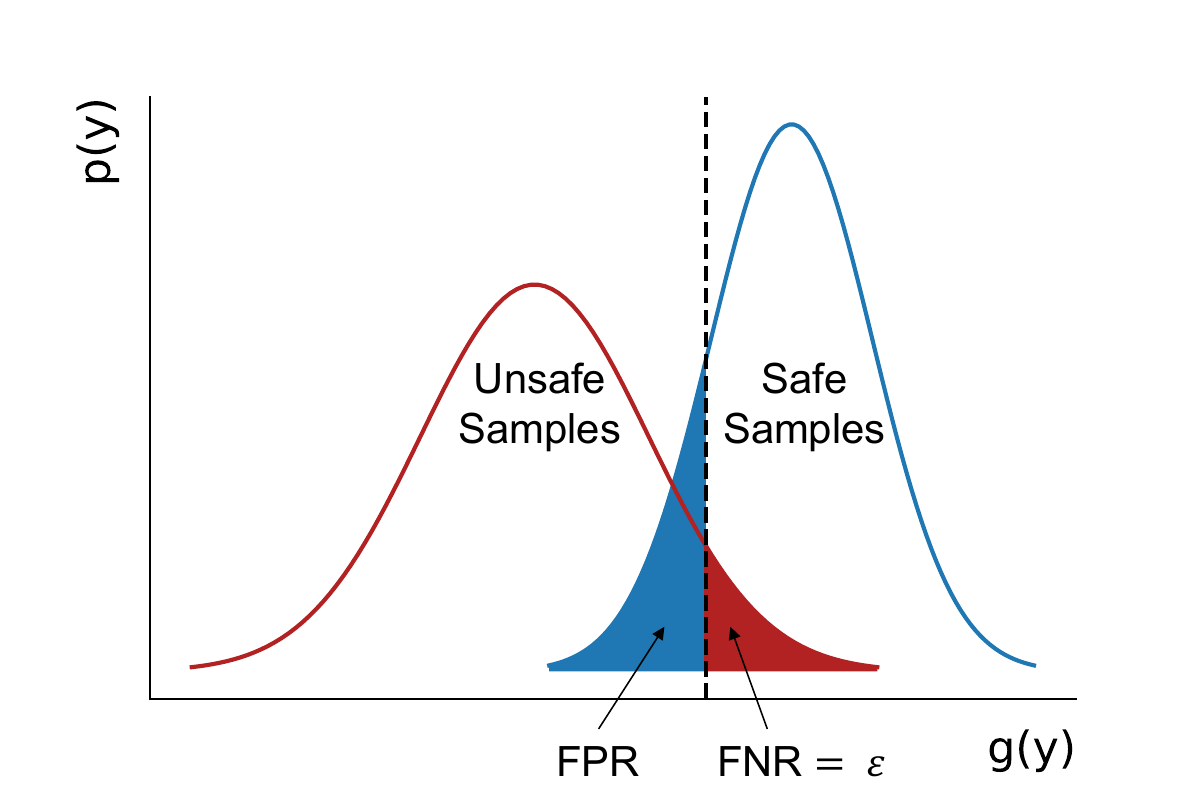}
    \caption{Even in the limit of infinite validation data, the best false positive rate achievable (for a given $\epsilon$-safety level) is determined by the distribution of the safe samples and the unsafe samples under the surrogate safety score function $g$.}
    \label{fig:optimal_fpr}
\end{figure}

\subsubsection{Finite Validation Data Regime} The lack of sufficient validation data is another source of error that degrades the best achievable FNR/FPR trade-off. Intuitively, because we need to \textit{provably guarantee} the FNR for $\epsilon$-safety, in the absence of sufficient validation data, we must be conservative and issue more warnings than necessary. For example, with zero validation data we have no choice but to issue a warning for nearly every example, leading to a high FPR. 
In fact, we show in Proposition~\ref{prop:fpr_lower_bound} in Section~\ref{sec:algorithm} that if we have fewer than $T$ data samples, we cannot guarantee better than $O(1/T)$ FNR without incurring an FPR of close to $1$ for any distribution-free warning function that depends only on the ordering of the $g(Y)$ values. 

Our conformal algorithm can guarantee an $O(1/T)$ FNR while the FPR is not much higher than in the infinite data regime, demonstrating the (asymptotic) optimality of the conformal algorithm presented in Algorithm~\ref{alg:general2}.

\subsection{Algorithm to Achieve Guaranteed Safety Assurances}
\label{sec:algorithm}

In this section, we will describe an algorithm that achieves $\epsilon^*$-safety with a low (nontrivial) false positive rate. The setup is as described in Section \ref{sec:setup}, with current observations $X$, simulator $Y$, and true unknown future states $Z$.

If the simulation $Y$ is perfect and has the same type as the ground truth future state $Z$, i.e. $Z = Y$ almost surely, then we can simply set $g = f$ and choose $w(Y) = \mathbb{I}(g(Y) < f_0)$, and this $w$ will automatically satisfy our definition of safety. However, in most applications, it is difficult to provide any guarantees on the accuracy of the simulation. For example, in autonomous driving situations, traffic participants can behave in unexpected and hard to predict ways. 

When we are uncertain about the simulation accuracy, we will require an additional training dataset. With a dataset of (simulated future state, true future state) pairs  $(Y_1, Z_1)$, $(Y_2, Z_2), \cdots, (Y_T, Z_T)$, where $T$ is the number of samples, we can guarantee $\epsilon$-safety. Let $(\hY, \hZ)$ denote a new test sample. We require only a single assumption on the dataset:

\begin{assumption} \label{assump:exchangeable}  
The sequence $(Y_1, Z_1)$, $(Y_2, Z_2)$, $\cdots$, $(Y_T, Z_T)$, $(\hY, \hZ)$ is exchangeable, i.e. the probability of observing any permutation of the sequence is equally likely.
\end{assumption}

Exchangeability is a strong assumption. However, it is weaker than typical i.i.d. assumptions that underlie most machine learning methods with performance guarantees: if a sequence of data is i.i.d., then it is also exchangeable. 
In addition, if the distribution shifts, it is not prohibitively costly to collect a new training dataset from the shifted distribution. This is because we require only a very small dataset (e.g. in most of our experiments, the training dataset contains only about 50 examples of unsafe situations).

Based only on Assumption~\ref{assump:exchangeable}, we design an algorithm to guarantee safety on test data. 
The algorithm can be thought of as an instantiation of the conformal prediction framework. 

\begin{algorithm}
\caption{Approximate $\epsilon$-safety}
\begin{algorithmic}[1]
\Require{Training dataset $(Y_1, Z_1), (Y_2, Z_2), \cdots, (Y_T, Z_T)$, Surrogate safety score $g$, True safety score $f$, Threshold $f_0$; A new simulation $\hY$}
\Ensure{$\{0, 1\}$}
\hrulealg
\State Compute 
\begin{equation*}
    \Ac = \big\lbrace g(Y_t) \ssep f(Z_t) < f_0, t= 1, \cdots, T \big\rbrace.
\end{equation*}
\State Sample $U$ uniformly from
\begin{equation*}
    U \sim \big\lbrace 0, 1,\cdots, |\lbrace a \in \Ac \ssep a = g(\hY) \rbrace | \big\rbrace.
\end{equation*}
\State Compute 
\begin{equation*}
    q = \frac{\ | \lbrace a \in \Ac \ssep a < g(\hY) \rbrace | + U   + 1}{|\Ac|+1}.
\end{equation*}
\State If $q \leq 1 - \epsilon$ then \textbf{output} 1, otherwise \textbf{output} 0 \label{alg2:line1}.
\end{algorithmic}
\label{alg:general2} 
\end{algorithm} 

Intuitively, the procedure is as follows. We first compute a predicted safety score (based on the simulator outputs) for each unsafe sample in the training dataset (Line 1). We then sample a number from a uniform distribution between 0 and $|\lbrace a \in \Ac \ssep a = g(\hY) \rbrace |$, where $|\lbrace a \in \Ac \ssep a = g(\hY) \rbrace |$ is the number of unsafe training data points with surrogate safety scores that are equal to the surrogate safety score of the new simulation $g(\hY)$ (Line 2). \footnote[2]{This randomization factor $U$ is simply a small correction factor that guarantees exact coverage in case of ties (see~\citet{Vovk2005Algorithmic}).} We next compute the quantile value for the new test simulation, i.e.\ the proportion of validation samples with a lower safety score than $\hY$ (Line 3), with a small randomization factor from the previous step. If this quantile value is smaller than $1 - \epsilon$ (i.e. fewer than $1 - \epsilon$ of the unsafe samples from the training set have a lower safety score), we say that this may be an unsafe situation and issue an alert; otherwise, we say that it is safe (Line 4). 

The following proposition shows that Algorithm~\ref{alg:general2} can guarantee safety. 

\begin{prop}
\label{prop:safety2} 
Algorithm~\ref{alg:general2} is $\epsilon + 1/(1+|\Ac|)$-safe (with respect to $\hY$ and $\hZ$), under Assumption 1.
\end{prop}

\begin{proof}
Given a dataset 
\begin{equation*}
    (Y_1, Z_1), (Y_2, Z_2), \cdots, (Y_T, Z_T),
\end{equation*}
where each data point contains the (simulated future state, true future state), denote the subset of ``unsafe'' data as 
\begin{equation*}
    (Y_{c_1}, Z_{c_1}), (Y_{c_2}, Z_{c_2}), \cdots, (Y_{c_M}, Z_{c_M}),
\end{equation*}
where $Z_{c_t}$ is the $t$-th unsafe example (i.e. $f(Z_{c_t}) < f_0$).  For typographical clarity throughout this proof, we use $M$ as shorthand to represent the number of unsafe data points,
\begin{equation*}
M = |\Ac| = \big| \big\lbrace t \ssep f(Z_t) < f_0  \big\rbrace \big|,
\end{equation*}
and we use $\unsafeDataLowScoreCount$ to represent the number of unsafe data points with surrogate safety score less than $g(\hat{Y})$,
\begin{align*}
\unsafeDataLowScoreCount &= \big| \big\lbrace t \ssep g(Y_{c_t}) < g(\hat{Y}) \big\rbrace \big| 
= \big| \big\lbrace a \in \Ac \ssep a < g(\hY) \big\rbrace \big|.
\end{align*}

Suppose that $\hat{Z}$ is also unsafe, i.e. $f(\hat{Z}) < f_0$. Let $\lbag \cdot \rbag$ denote an unordered bag (i.e. it is a set that can have repeated elements).  We use $B$ to represent the unordered bag of unsafe data,
\begin{equation*}
\unsafeBag = \lbag (Y_{c_1}, Z_{c_1}), \cdots, (Y_{c_M}, Z_{c_M}), (\hat{Y}, \hat{Z}) \rbag.
\end{equation*}

To bound the safety, we first note that the probability of a false negative is given by
\begin{flalign*}
    \Pr[w(\hat{Y}) = 0] & = \Pr[q > 1 - \epsilon] \\
    &= \Eb\Bigl[ \Pr\bigl[q > 1 - \epsilon \big| \unsafeBag \bigr] \Bigr] \hspace{5.1em} \text{(Tower)}\\
    &= \Eb\Bigl[ \Pr\bigl[ \unsafeDataLowScoreCount + U > (1 - \epsilon)(M+1) - 1 \big| \unsafeBag \bigr]\Bigl] \\
    && \mathllap{ \text{(Definition)} \hspace{1.2em} }
\end{flalign*}

By the assumption of exchangeability, we are equally likely to observe any permutation of $\unsafeBag$. Intuitively, $g(\hat{Y})$ is equally likely to be the largest, 2nd largest, etc., among $g(Y_{c_1}), \cdots, g(Y_{c_M}), g(\hat{Y})$. Formally, the random variable $N + U$ takes on all values $\lbrace 0, 1, \cdots, M \rbrace$ with equal probability.\footnote[3]{Note that the maximum value of $U$ is the number of unsafe training data points with surrogate safety scores that are equal to the surrogate safety score of the new simulation, and so $| \lbrace a \in \Ac \ssep a < g(\hY) \rbrace | + |\lbrace a \in \Ac \ssep a = g(\hY) \rbrace | = M $.} Therefore, 

\begin{align*}
    &\Pr\biggl[ \unsafeDataLowScoreCount + U > (1 - \epsilon)(M+1) - 1 \Big| \unsafeBag \biggr] \\
    & = 1 - \Pr\biggl[ \unsafeDataLowScoreCount + U \leq (1 - \epsilon)(M+1) - 1 \Big| \unsafeBag \biggr] \\
    & \leq 1 - \frac{\lceil(1 - \epsilon)(M+1)-1 \rceil}{M+1}  \\
    & = \frac{\lfloor M+1 - (1-\epsilon)(M+1) + 1 \rfloor}{M+1}  \\
    & \leq \frac{1 + \epsilon M + \epsilon}{M+1} \\
    & = \epsilon + \frac{1}{M+1}.
\end{align*}

We can combine this result with the previous statement of the probability of a false negative to get 
\begin{align*}
    &\Pr[w(\hat{Y}) = 0] \\
    &= \Eb\Biggl[ \Pr\biggl[ \unsafeDataLowScoreCount + U >  (1 - \epsilon)(M+1) - 1 \Big| \unsafeBag \biggr] \Biggr] \\
    &\leq \Eb\left[\epsilon + \frac{1}{M+1} \right]  \\
    &= \epsilon + \frac{1}{M+1} 
\end{align*}
as required. \qed

\end{proof}

To use Proposition~\ref{prop:safety2} to provide safety guarantees, we choose $\epsilon$ based on the number of samples available $|\Ac|$. Specifically, if the desired safety level is $\epsilon^*$, then we can choose any $\epsilon < \epsilon^*$ in Algorithm~\ref{alg:general2} such that 
\begin{align} 
\epsilon + 1/(1+|\Ac|) \leq \epsilon^*\label{eq:sample_complexity1}
\end{align} 
In other words, if our choice of $\epsilon$ satisfies Eq. (\ref{eq:sample_complexity1}), then Algorithm~\ref{alg:general2} will be $\epsilon^*$-safe. 
Intuitively, choosing a large $\epsilon$ decreases the false positive rate (FPR). This is because according to Algorithm~\ref{alg:general2} Line \ref{alg2:line1}, choosing a larger $\epsilon$ decreases the number of times that a warning is output. Therefore, based on the number of samples in the training dataset $|\Ac|$, we  choose the largest $\epsilon$ that satisfies Eq. (\ref{eq:sample_complexity1}); i.e. we choose 
\[\epsilon = \epsilon^* - 1/(1+|\Ac|). \] 
We will call $1/(1+|\Ac|)$ the discretization error. 

Proposition~\ref{prop:safety2} also reveals the sample complexity of the conformal prediction algorithm. If the number of unsafe examples is too small ($|\Ac| \leq 1/\epsilon^* - 1$), then we must choose $\epsilon < 0$ to ensure $\epsilon^*$-safety according to Proposition~\ref{prop:safety2}. Algorithm~\ref{alg:general2} with $\epsilon < 0$ will trivially always output $1$ (i.e. always issue a warning). On the other hand, if the number of unsafe examples exceeds the threshold ($|\Ac| > 1/\epsilon^* - 1$), then there will be an $\epsilon > 0$ that ensures $\epsilon^*$-safety according to Proposition~\ref{prop:safety2}. Consequently, Algorithm~\ref{alg:general2} will not be trivial. In practice, we find that to get good results and a low false positive rate, it is sufficient to have sample count $|\Ac|$ that exceeds the threshold by a small margin, such as $|\Ac| = 1.5/\epsilon^*-1$. For example, to achieve a 5\% false positive rate, it is sufficient to have only about 30 unsafe examples.

Proposition~\ref{prop:fpr_lower_bound} demonstrates that Algorithm~\ref{alg:general2} is asymptotically optimal, since it can guarantee an $O(1/T)$ FNR while the FPR is not much higher than in the infinite data regime.

\begin{prop}
\label{prop:fpr_lower_bound} 
There is no $\epsilon$-safe warning system based on the ordering of $g(Y)$ values that can achieve a false positive rate lower than $1-(1+T)\epsilon$.
\end{prop}

We conjecture that Proposition~\ref{prop:fpr_lower_bound} holds for all functions, but we prove it for the fairly general class of functions specified by Equation~\ref{eq:test_function_form} below, comparing the surrogate safety score for our new test sample $g(\hY)$ relative to the $g(Y_1), \cdots, g(Y_T)$ values.

\begin{proof}

Consider a function $w$ that maps a dataset $\Dc = (g(Y_1), Z_1), \cdots, (g(Y_T), Z_T)$ of unsafe examples, and a new data point $g(\hY)$, to $\lbrace 0, 1\rbrace$. Let $w$ be a warning function that gives a distribution-free false negative rate guarantee that depends only on the ordering between $g(Y_1), \cdots, g(Y_T), g(\hY)$ (rather than on their specific values). In other words, $w$ takes the form defined by 
\begin{flalign}
    w(\Dc, \hY) = \left\lbrace \begin{array}{ll} \phi \left( \big| \big\lbrace t \ssep g(\hY) < g(Y_t) \big\rbrace \big| \right), \\
    & \mathllap{ \text{ with probability } \gamma } \\
    1, & \mathllap{ \text{ with probability } 1-\gamma } \end{array}  \right. \label{eq:test_function_form} 
\end{flalign}
for some deterministic function $\phi$ and real number $\gamma$. We know that when the data is exchangeable, the random variable $| \lbrace t \ssep g(\hY) < g(Y_t) \rbrace |$ is uniformly distributed on $\lbrace 0, 1, \cdots, T \rbrace$.

\textbf{Case 1:} Suppose $\phi$ outputs the value $0$ (no warning) for at least one possible input. Then the false negative rate is given by
\begin{align}
    \text{FNR} \geq \gamma/(1+T), 
\end{align}
and the false positive rate is given by
\begin{align}
    \text{FPR} \geq 1-\gamma, 
\end{align}
so combined we have 
\begin{align}
    \text{FPR} \geq 1-\gamma \geq 1 - (1+T) \text{FNR} \geq 1 - (1+T) \epsilon.
\end{align}

\textbf{Case 2:} Suppose $\phi$ outputs the value $0$ for none of the inputs (i.e. it always issues a warning). Then the false negative rate and false positive rate are given by 
\begin{align}
    \text{FNR} = 0, \text{FPR} = 1,
\end{align}
so we would still (trivially) have $\text{FPR} \geq 1-(1+T)\epsilon$. 

Thus, if $w$ takes the form of Equation~\ref{eq:test_function_form}, then the false positive rate must be lower bounded by $1-(1+T)\epsilon$. In other words, when $\epsilon = o(1/T)$, the false positive rate tends to $1$ when $T$ is large.  
\qed 

\end{proof}

\subsubsection{Algorithm Design Choices} There are two major components of our algorithm that can be tuned to achieve a tighter FPR: the surrogate safety score, and the trained simulator or prediction model. A surrogate safety score that is better correlated with the true safety score will lead to a better FPR, as will a more accurate simulator model. Note that patterns of inaccuracies in the simulator model will also lead to patterns of errors in the false alarms. For example, an autonomous driving simulator that is particularly inaccurate around yield signs will lead to surrogate safety scores that are not well correlated with the true safety score around yield signs; more alarms will be issued in order to maintain the specified FNR and thus the FPR will be higher.

The guarantees and analysis of Algorithm~\ref{alg:general2} will hold regardless of the surrogate safety score and simulator model used.  
However, if these components are chosen poorly, not enough data is available, or the required $\epsilon^*$ is too stringent, then this procedure could become trivial (e.g. always issuing a warning). In practice however, we find that we are able to obtain good results for an autonomous driving application with an off-the-shelf prediction model for reasonably low $\epsilon^*$-values with not too much data  (see Section~\ref{sec:experiments-driving}).

\subsection{Comparing Conformal Prediction with PAC Learning} 
We further compare the statistical learning and conformal prediction guarantees. We first clarify the notation and formally define the different assumptions. Consider a sequence of training data $(Y_1, Z_1), (Y_2, Z_2), \cdots, (Y_T, Z_T)$ and a sequence of test data $(\hY_1, \hZ_1), (\hY_2, \hZ_2), \cdots, (\hY_N, \hZ_N)$. 
Let $c_1, \cdots, c_M$ denote the unsafe subsequence of test data, i.e. 
$
    (\hY_{c_1}, \hZ_{c_1}), (\hY_{c_2}, \hZ_{c_2}), \cdots, (\hY_{c_{M}}, \hZ_{c_M})
$
is the subsequence of $(\hY_1, \hZ_1), (\hY_2, \hZ_2), \cdots, (\hY_N, \hZ_N)$ such that, for all $m$, $f(\hZ_{c_m}) < f_0$. 

Two possible assumptions that we could make on the training and test data sequences are shown in Assumptions~\ref{assump:exchangeable2} and~\ref{assump:iid}.
In particular, marginal exchangeability (Assumption~\ref{assump:exchangeable2}) is the same as Assumption~\ref{assump:exchangeable} from the previous section. The only difference here is that we explicitly state that we only require exchangeability with \textit{each} test data point. 

\begin{assumption}[Marginal exchangeability]\label{assump:exchangeable2} 
For each $n$, the sequence $(Y_1, Z_1)$, $(Y_2, Z_2)$, $\cdots$, $(Y_T, Z_T)$, $(\hY_n, \hZ_n)$ is exchangeable. 
\end{assumption} 

\begin{assumption}[Independent and identically distributed]\label{assump:iid}
The training / test data sequences $(Y_1, Z_1)$, $(Y_2, Z_2)$, $\cdots$, $(Y_T, Z_T)$, $(\hY_1, \hZ_1), (\hY_2, \hZ_2), \cdots, (\hY_N, \hZ_N)$ are drawn from an i.i.d. distribution. 
\end{assumption} 

Given a warning function, we use the random variables $\hF_{c_1}, \cdots, \hF_{c_M}$ to denote failure of the warning function, i.e. $\hF_{c_m} = \Ibb(w(\hY_{c_m}) = 0)$.
Note that $\hF_{c_m}$ depends on $w$, but we drop this dependence from our notation. 

A learning algorithm is a function that takes as input the training data $(Y_1, Z_1), (Y_2, Z_2), \cdots, (Y_T, Z_T)$ and outputs a warning function $w: X \to \lbrace 0, 1\rbrace$. There are two main paradigms for designing learning algorithms with guarantees. 

\textbf{PAC Learning:} Under Assumption~\ref{assump:iid}, a learning algorithm is $(\epsilon,\delta)$-safe if with $1-\delta$ probability (with respect to randomness of the training data) the learned warning function $w$ satisfies for some $\epsilon' < \epsilon$
\begin{align*}
\hF_{c_1}, \cdots, \hF_{c_M} \sim \mathrm{Bernoulli}(\epsilon').  \numberthis\label{eq:pac_guarantee} 
\end{align*}

\textbf{Conformal Learning:} For completeness we restate the conformal learning guarantee. A learning algorithm is $\epsilon$-safe if the learned function $w$ satisfies for some $\epsilon' < \epsilon$
\begin{align*}
    \hF_{c_m} \sim \mathrm{Bernoulli}(\epsilon'), \text{ for all } m = 1, \cdots, M. \numberthis\label{eq:conformal_guarantee} 
\end{align*}

\subsubsection{Comparing Assumptions} Conformal learning requires weaker assumptions. 
Assumption~\ref{assump:exchangeable2} is much weaker than Assumption~\ref{assump:iid}, and hence is applicable to a much larger class of problems. For example, consider an autonomous driving application where the training data are snippets from randomly sampled driving scenes (no two training data points come from the same driving scene), and the test data $(\hY_1, \hZ_1), (\hY_2, \hZ_2), \cdots, (\hY_N, \hZ_N)$ 
is a sequence of driving snippets from a random driving scene. The test data points are not independent because they are from the same scene, and hence Assumption~\ref{assump:iid} is violated. However, Assumption~\ref{assump:exchangeable2} holds because the training data and any \textit{single} test sample are snippets from randomly sampled driving scenes. 

\subsubsection{Comparing Sample Complexity} Conformal learning requires $\Theta(1/\epsilon)$ training examples of unsafe situations (Proposition~\ref{prop:safety2}), while standard analysis in PAC learning requires $\Theta(1/\epsilon^2)$ examples. For example, consider the following $(\epsilon,\delta)$-safe algorithm: based on the simulation $Y$ and the surrogate safety function $g$, we consider the family of warning functions $w_\theta(Y) = \Ibb(g(Y) < \theta)$. Our goal is to estimate the false negative rate of $w_\theta$ (denoted by $\epsilon^*(\theta)$) for each $\theta$ and select the smallest $\theta$ such that $\epsilon^*(\theta) \leq \epsilon$.  

To estimate $\epsilon^*$, we compute the (empirical) false negative rate (denoted by $\hat{\epsilon}(\theta)$) on the training data, i.e. 
\begin{align}
    \hat{\epsilon}(\theta)  = 1/M \sum_m \Ibb( g(Y_{c_m}) \geq \theta) 
\end{align}
and use a standard concentration inequality (such as Hoeffding) to bound the difference between $\epsilon^*(\theta)$ and $\hat{\epsilon}(\theta)$. Specifically, with probability $1-\delta$ 
\begin{align}
   \epsilon^*(\theta) \in \hat{\epsilon}(\theta) \pm \sqrt{\frac{ \log(1/\delta)}{2M}}. \label{eq:hoeffding}
\end{align}
Note that Eq. (\ref{eq:hoeffding}) is already the tightest bound possible up to constants~\citep{Foody2009Sample}. To verify that $w_\theta$ has a false negative rate that is less than or equal to $\epsilon$, we have to check that 
\begin{align*} 
\hat{\epsilon}(\theta) + \sqrt{\frac{ \log(1/\delta)}{2M}} \leq \epsilon,
\end{align*} 
which requires
\begin{align*} 
\sqrt{\frac{ \log(1/\delta)}{2M}}  \leq \epsilon \iff M \geq \frac{\log(1/\delta)}{2\epsilon^2}.
\end{align*} 
This means that we must have at least $\Theta(1/\epsilon^2)$ samples. 
 
In words, even a fixed $w_\theta$ requires $\Theta(1/\epsilon^2)$ samples to verify its false negative rate according to Eq. (\ref{eq:hoeffding}). Thus, finding $w_\theta$ to provably achieve low false negative rate should require at least as many, if not more, training examples.

\subsubsection{Comparing Usefulness of Guarantees} 
PAC learning and conformal learning both have advantages. PAC learning has the advantage that its i.i.d. error rate guarantee in Eq. (\ref{eq:pac_guarantee}) is stronger than the marginal error rate guarantee in  Eq. (\ref{eq:conformal_guarantee}). For example, if the downstream user is very sensitive to high variance (i.e. it is unacceptable for all test examples to fail simultaneously even if the probability is vanishingly small) then the i.i.d. error rate guarantee in Eq. (\ref{eq:pac_guarantee}) might be necessary. Nevertheless, the risk can be reduced by alternative methods such as financial tools (insurance). 
On the other hand, the conformal learning guarantee in Eq. (\ref{eq:conformal_guarantee}) has the advantage that it always holds, while the PAC learning guarantee in Eq. (\ref{eq:pac_guarantee}) only holds with $1-\delta$ probability.

To summarize, conformal learning requires much weaker assumptions and fewer samples, and its guarantees always hold (rather than with $1-\delta$ probability). PAC learning offers stronger guarantees when its assumptions and sample complexity requirements are met.

\section{Experiments: Driver Alert System}
\label{sec:experiments-driving}

We empirically validate the guarantees of our framework on a driver alert safety system using real driving data. The system should warn the driver if the driver may get into an unsafe situation, without issuing too many false alarms. We show that the false negative rate (the percentage of unsafe situations that the system fails to identify) is indeed bounded according to Proposition~\ref{prop:safety2}, while the FPR remains low. 

\subsection{Experimental Setup}

\subsubsection{Methods} We evaluate our framework on the setup described in Section~\ref{sec:setup}. 
We use Trajectron++~\citep{Salzmann2020Trajectron} as our future dynamics model (i.e. in the notation of Section~\ref{sec:setup}, \ $Y$ is the output of Trajectron++ and $g = f$). We choose the safety score $f$ as a weighted distance metric, where agents in the direction of the ego-vehicle velocity vector are considered ``closer'' than agents in the orthogonal direction. 

More specifically, we define the safety score by the Mahalanobis distance between the ego-vehicle and the agent, where the first eigenvector is aligned with the ego-vehicle's velocity vector, and the second eigenvector is orthogonal to the ego-vehicle (see Figure~\ref{fig:mahalanobis}); the magnitude of the first eigenvector is the magnitude of the velocity, and the magnitude of the second eigenvector is approximately half of a car width (we use 1m). Intuitively, this means that agents that are along the ego-vehicle's velocity vector appear closer than agents in the perpendicular direction. 

\begin{figure}[h]
    \centering
    \includegraphics[width=0.9\linewidth]{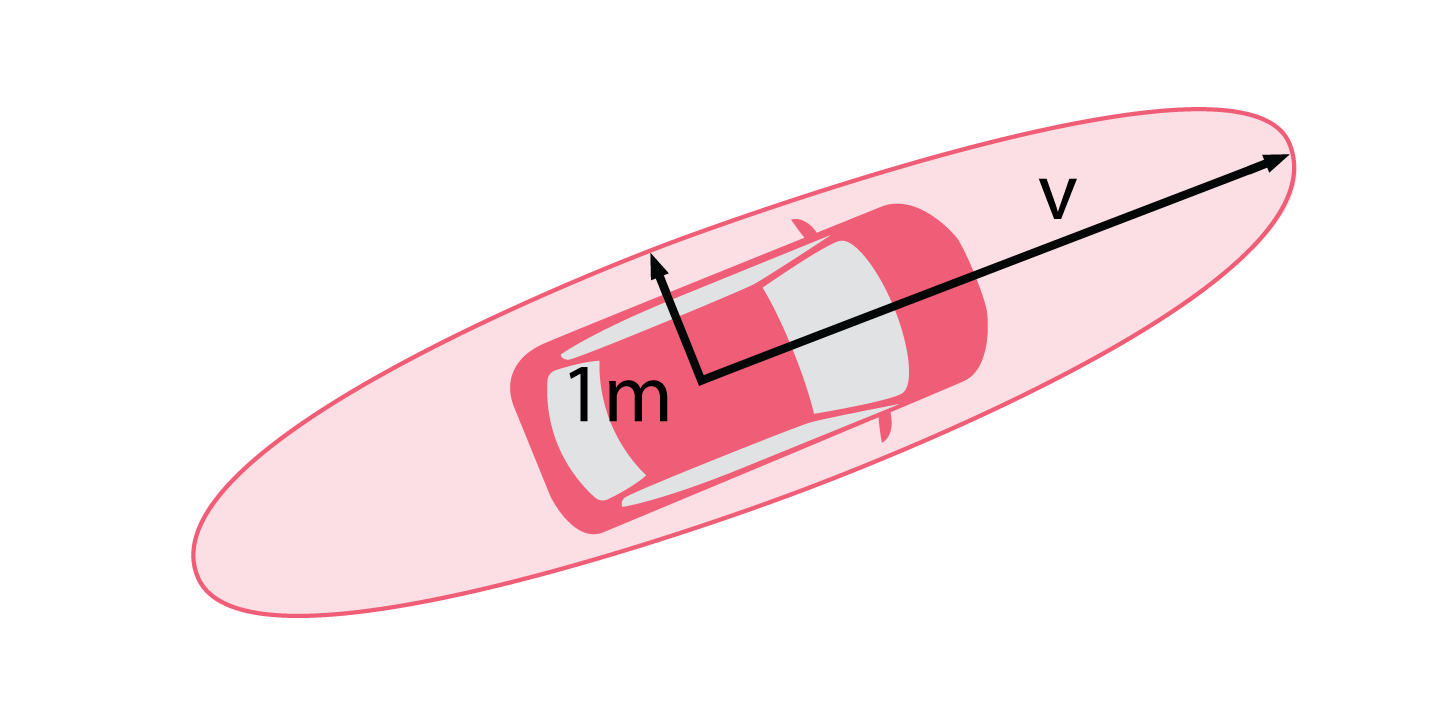}
    \caption{The Mahalanobis distance is transformed along these axes, such that agents in the direction of the velocity vector appear closer than agents in the perpendicular direction.}
    \label{fig:mahalanobis}
\end{figure}

\begin{figure*}[tb]
    \centering 
    \begin{subfigure}[t]{0.475\textwidth}
        \centering
        \includegraphics[width=\textwidth]{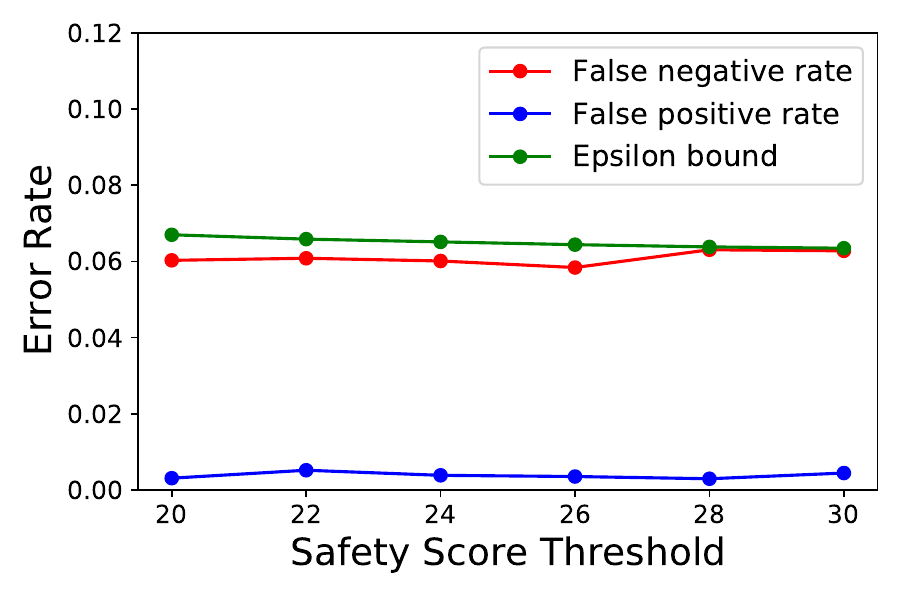}
        \caption{False negative and false positive rates on the nuScenes dataset with varying $f_0$ and $\epsilon = 0.05$. The theoretical upper bound on epsilon is shown in {\color{OliveGreen} green}. The false negative rate ({\color{red} red}) is below the upper bound and the false positive rate ({\color{blue} blue}) is very low.}
        \label{fig:vary-threshold}
    \end{subfigure}
    \hfill
    \begin{subfigure}[t]{0.475\textwidth}
        \centering
        \includegraphics[width=\textwidth]{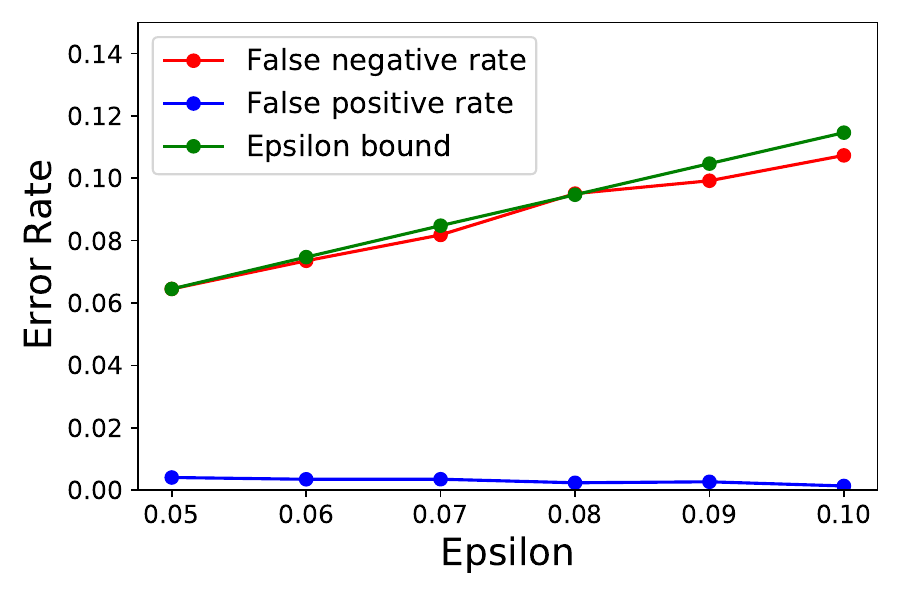}
        \caption{False negative and false positive rates on the nuScenes dataset with varying $\epsilon$ and $f_0 = 25$. The theoretical upper bound on epsilon is shown in {\color{OliveGreen} green}. The false positive rate ({\color{blue} blue}) improves with higher $\epsilon$. }
        \label{fig:vary-epsilon}
    \end{subfigure}
    \vskip\baselineskip
    \begin{subfigure}[t]{0.475\textwidth}
        \centering
        \includegraphics[width=\textwidth]{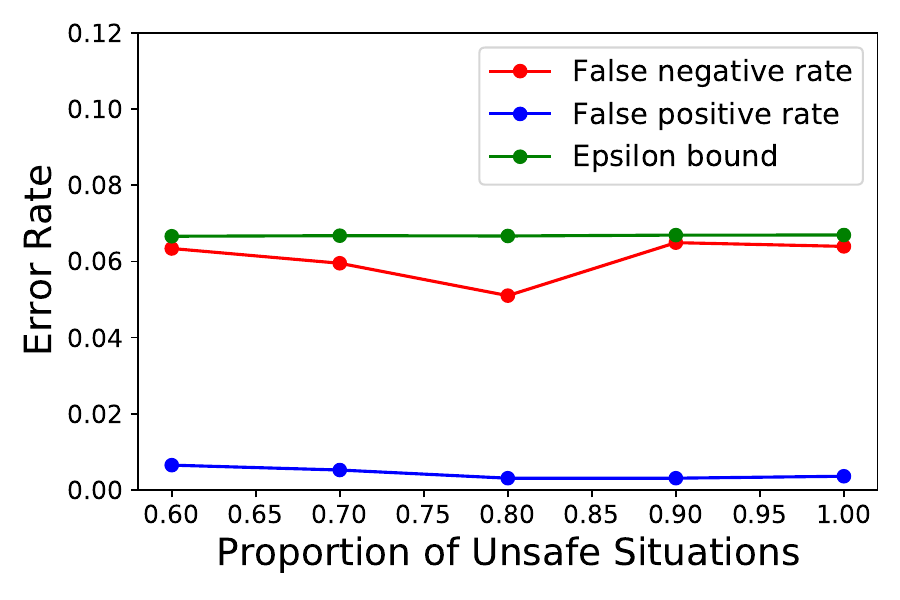}
        \caption{False negative and false positive rates on the nuScenes dataset with varying proportions of unsafe samples in the training set. The theoretical upper bound on epsilon is shown in {\color{OliveGreen} green}. Here, $\epsilon = 0.05$ and $f_0 = 25$.}
        \label{fig:vary-label-freq}
    \end{subfigure}
    \hfill
    \begin{subfigure}[t]{0.475\textwidth}
        \centering
        \includegraphics[width=\textwidth]{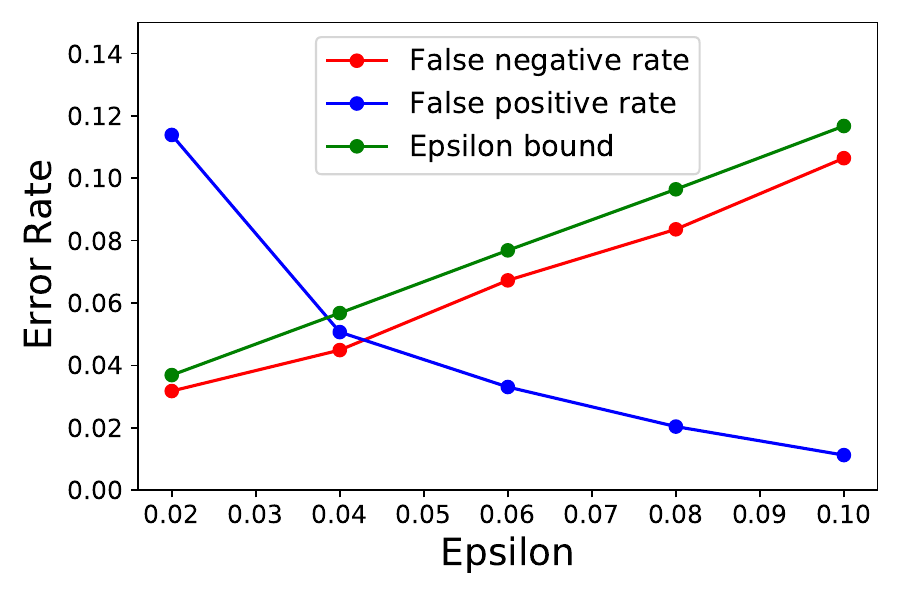}
        \caption{False negative and false positive rates on the Kaggle Lyft dataset with a varying $\epsilon$ value. Here, $f_0 = 3.5$, and there are approx. 50 unsafe examples in the training dataset.}
        \label{fig:lyft}
    \end{subfigure}
    \caption{False negative rate, false positive rate, and theoretical upper bound on $\epsilon$ for the nuScenes and Lyft datasets while varying several parameters.}
\end{figure*}

\subsubsection{Datasets} We use the nuScenes~\citep{nuscenes2019} and the Kaggle Lyft Motion Prediction~\citep{kaggleLyft} autonomous driving datasets. 
Each dataset contains multiple scenes, and each scene contains multiple trajectories. The trajectories in a scene are correlated with each other, but the different scenes are sufficiently distinct from each other to be considered exchangeable. To generate a dataset of exchangeable trajectories, we sample a single trajectory uniformly at random from each scene. 

The nuScenes dataset includes 952 scenes collected across Boston and Singapore, divided into a 697/105/150 train/val/test split (the same split used for the original Trajectron++). Each scene is 20 seconds long. The Kaggle Lyft Motion Prediction dataset includes approximately 16k scenes, divided into an 70\%/15\%/15\% train/val/test split. Each scene is 25 seconds long. Both datasets include labeled ego-vehicle trajectories as well as labeled detections and trajectories for other agents in the scene. Note that for both of these datasets, because the training split was used to train the Trajectron++ model, we use the validation split as the input training data for Algorithm~\ref{alg:general2}. A visualization of trajectory predictions output by Trajectron++ is shown in Figure~\ref{fig:trajectron_output} \citep{Salzmann2020Trajectron}.

\begin{figure}[H]
    \centering
    \includegraphics[width=\linewidth]{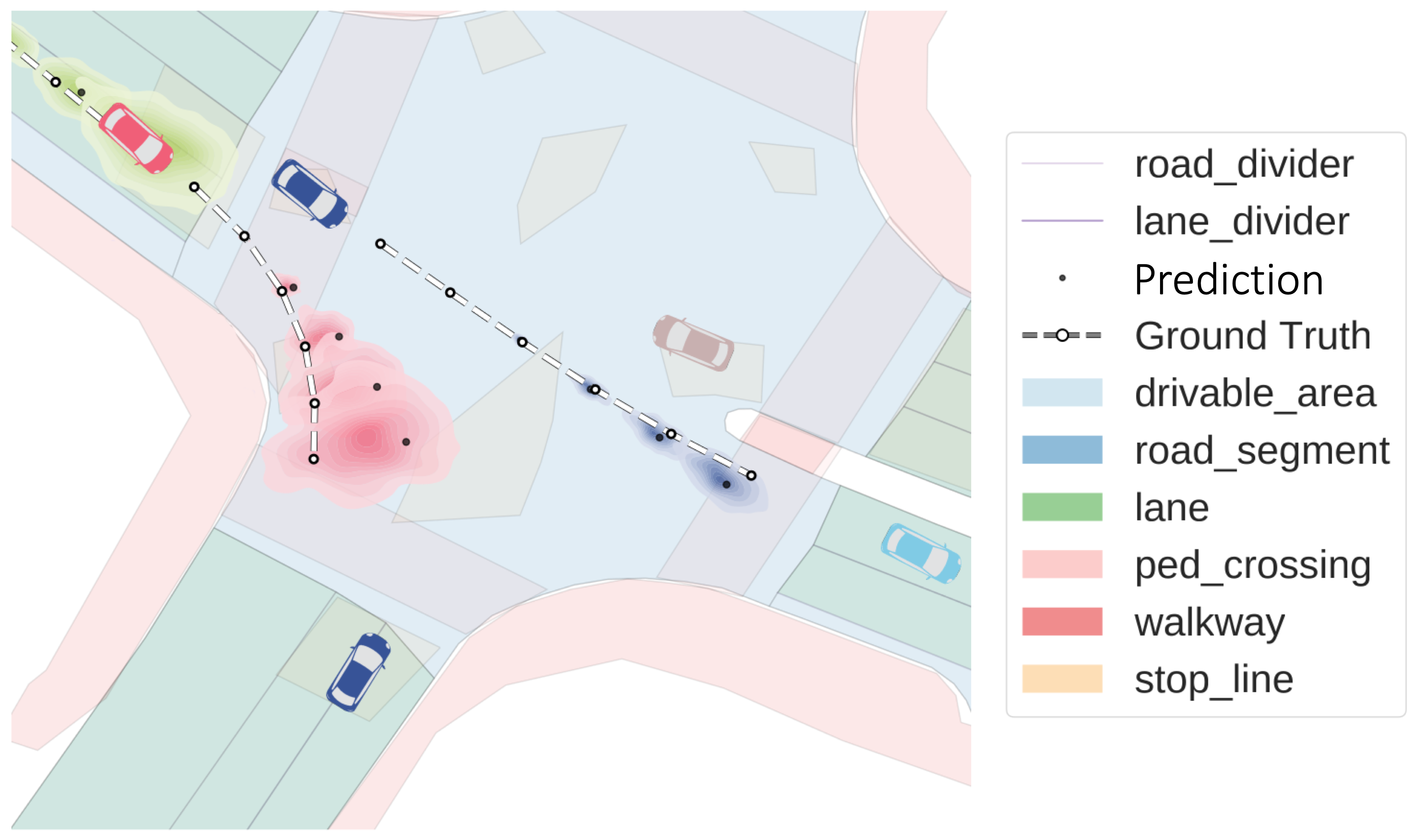}
    \caption{Visualization of trajectory predictions output by Trajectron++~\citep{Salzmann2020Trajectron}. }
    \label{fig:trajectron_output}
\end{figure}

\subsubsection{Data Splitting} To compute average performance, we use random train and test splits. 
For both datasets, we first pool together all available data points, randomly shuffle them, and separate them back into training and test splits (with the same size as the original splits). 
We ran 100 trials for each experiment, and averaged over the results. 

\subsection{Results and Discussion}

In Fig.~\ref{fig:vary-threshold}, \ref{fig:vary-epsilon}, and \ref{fig:vary-label-freq} we vary several parameters (safety threshold $f_0$, safety guarantee $\epsilon$, and proportion of unsafe situations) for nuScenes. We show qualitatively similar results for the Lyft dataset in Figure~\ref{fig:lyft}. Our main observations:

\begin{enumerate}
    \item The false negative rate (i.e. safety) is always within the theoretical bound in Proposition~\ref{prop:safety2}. We achieve these false negative rates with very little data. nuScenes has 50-70 unsafe examples in the training dataset, and Lyft has about 50. Yet, even with these few examples, we can ensure a false negative rate to within 1 or 2\% of the desired $\epsilon$. 
    \item The false positive rate (FPR) is generally very good --- well below 1\% on the nuScenes dataset. We use an off-the-shelf trajectory predictor trained on a small academic dataset; a more accurate trajectory predictor trained on industry-sized datasets might be expected to provide a more discriminative safety score (as in Figure~\ref{fig:optimal_fpr}), and thus a further improved FPR. Note that as shown in Figure~\ref{fig:fpr_vs_samples}, there is a tradeoff between $\epsilon$ and the FPR when there are few (e.g. $< 1/T$) samples, which is consistent with what our theory from Section~\ref{sec:fpr} would predict. Figure~\ref{fig:fpr_vs_samples} plots the epsilon bound as well as the false negative and false positive rates vs. the number of unsafe samples in the validation dataset; we see that when $\epsilon$ decreases as $1/T$, the false positive rate is relatively flat and low.
    \item One previously unmentioned benefit of our approach is that our method is robust to label frequency shift --- the frequency of unsafe situations can differ between the training data and test data. Observe that the output of Algorithm~\ref{alg:general2} depends only on the unsafe examples; consequently, the safety guarantee in Proposition~\ref{prop:safety2} still holds if we increase or decrease the number of safe examples. For example, the training data collection process could intentionally focus on unsafe situations, so that unsafe examples are over-represented in the training data. We empirically simulate this in Figure~\ref{fig:vary-label-freq} where we increase the proportion of unsafe examples in the training set (by deleting safe examples). The performance of our algorithm does not change qualitatively. 
\end{enumerate}

\begin{figure}[h]
    \centering
    \includegraphics[width=0.95\linewidth]{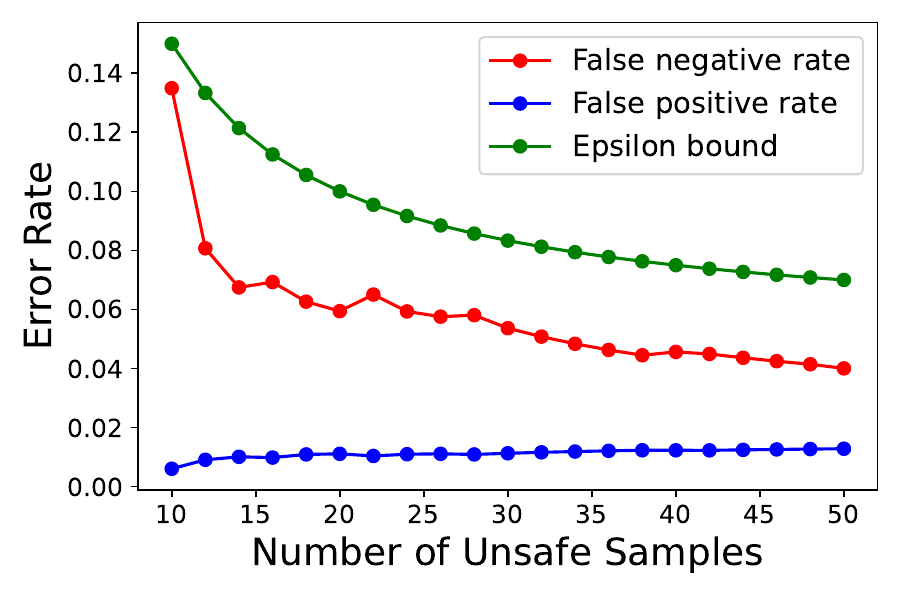}
    \caption{Epsilon bound, false negative rate, and false positive rate on the nuScenes dataset while varying the number of unsafe samples. Consistent with our theory from Section~\ref{sec:fpr}, the sum of $\epsilon$ and the false positive rate is high when there are few samples. }
    \label{fig:fpr_vs_samples}
\end{figure}

In Figure~\ref{fig:pac_error_rates}, we also plot the false negative and false positive rates at various $\epsilon$-values for the Lyft dataset using PAC estimates. 
In this experiment, we find a 99.9\% confidence bound for the $(1-\epsilon)$ quantile surrogate safety score among the unsafe samples, and issue an alert if the surrogate safety score of the new test sample is below this upper bound. Thus, with high (99.9\%) probability, the FNR will be less than or equal to $\epsilon$. As the figure demonstrates, the FPR is very high until $\epsilon$ is large, since there are too few samples to obtain a tight confidence interval.

\begin{figure}[h]
    \centering
    \includegraphics[width=0.95\linewidth]{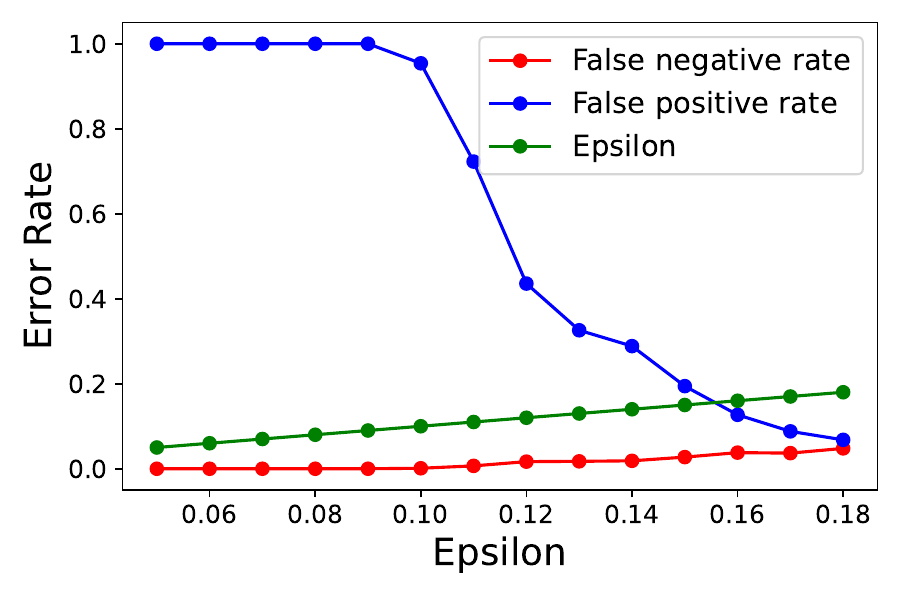}
    \caption{Epsilon, false negative rate, and false positive rate for the Lyft dataset using PAC confidence bounds. The false positive rate is very high until $\epsilon$ is large, because there are very few unsafe examples in the training dataset.}
    \label{fig:pac_error_rates}
\end{figure}

In the comparison of PAC learning and conformal learning, we argued that the main advantage of PAC learning is that the failures are i.i.d., so the total number of failures should have low variance (due to the Central Limit Theorem). However, we show empirically that users need not be overly concerned about highly correlated failures, as long as the test samples are not inherently highly correlated. 
We find that the variance on the false negative rate from different train/test splits is very low. With $\epsilon = 0.06$, for instance, it was only $0.0014$.
Table~\ref{tab:variance} displays the variance on the false negative rate calculated over the 100 trials for the Lyft dataset at each $\epsilon$ value. All of the variances are well below 0.003, suggesting that the test sequence false negative rates are clustered around $\epsilon$ (rather than having some sequences that fail on zero examples and others with catastrophic failures). As further evidence, in Figure~\ref{fig:box-plot}, we provide a representative box plot of the false negative rates over the 100 trials with $\epsilon = 0.04$. The false negative rate values are indeed clustered around $0.04$.

\begin{table}[h]
	\centering
	\setlength{\tabcolsep}{0.5em}
	\begin{tabular}{c | c c c c c}
        $\epsilon$ & $0.02$ & $0.04$ & $0.06$ & $0.08$ & $0.10$ \\ 
		\hline
        Variance     & 0.00096 & 0.0019 & 0.0014 & 0.0023 & 0.0024 \\
	\end{tabular}
	\vspace{2mm}
	\caption{\label{tab:variance}  Variance on the test sequence false negative rates at different $\epsilon$.}
\end{table}

\begin{figure}[h]
    \centering
    \includegraphics[width=0.9\linewidth]{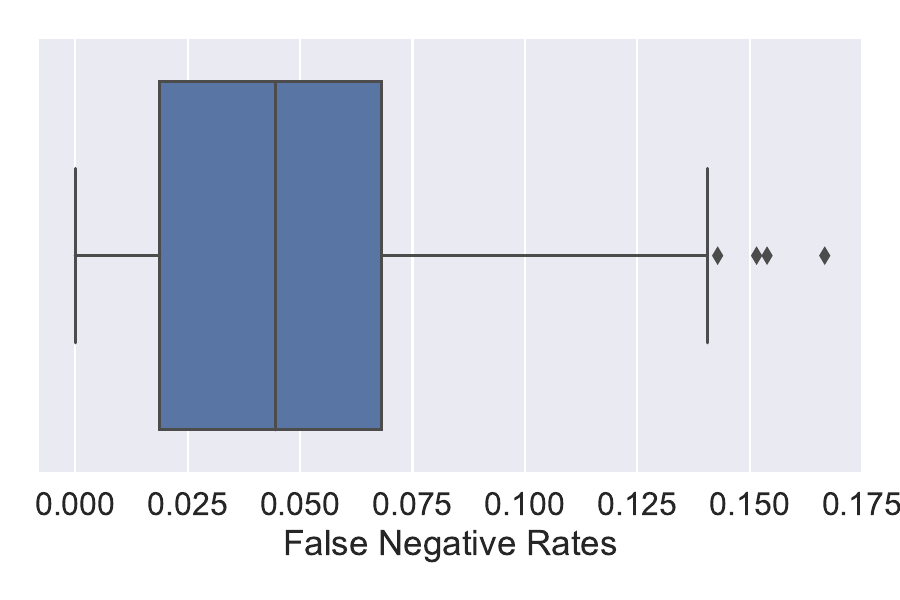}
    \caption{Box plot of the 100 false negative rates calculated over randomized train/test splits with $\epsilon = 0.04$.}
    \label{fig:box-plot}
\end{figure}

\section{Experiments: Robotic Grasping}
\label{sec:experiments-grasp}

Finally, we validate the guarantees of our framework on a robotic grasping system that should warn the user when the robot will fail to pick and transport an object. Picking is a core problem in warehouse robotics~\citep{correll2016analysis,eppner2016lessons,hernandez2016team,yu2016summary,zeng2017multi,mahler2017dex,mahler2017suction,mahler2019learning}, and failures hurt throughput (potentially even stopping the assembly line). Failures can also lead to dropped or damaged goods.

\subsection{Experimental Setup}
We again evaluate our framework on the setup described in Section~\ref{sec:setup}, using an open source dataset and model.  We use the Grasp Quality Convolutional Neural Network (GQ-CNN) from~\citet{mahler2017dex,mahler2017suction,mahler2019learning} as our predictor model and the DexNet 4.0 dataset of synthetic objects grasped with a parallel-jaw gripper~\citep{mahler2019learning}. (An example of objects from this dataset is shown in Figure~\ref{fig:dexnet_input}.)
The inputs to the GQ-CNN are a point cloud representation of an object, $\mathbf{y}$, and a candidate grasp, $\mathbf{u}$.  The GQ-CNN outputs the predicted probability, $Q_{\theta}(\mathbf{y}, \mathbf{u})$, that the candidate grasp will be able to successfully pick and transport the object.
We use this predicted probability as the surrogate safety score, $g = Q_{\theta}(\mathbf{y}, \mathbf{u})$.  We consider a candidate grasp ``unsafe'' if it will not be able to successfully pick the object (i.e.\ the true label is $Z = 0$). Note that this is exactly the ROC curve threshold tuning setup (with an additional guarantee on the false negative rate). 

\begin{figure}[h]
    \centering
    \includegraphics[width=0.7\linewidth]{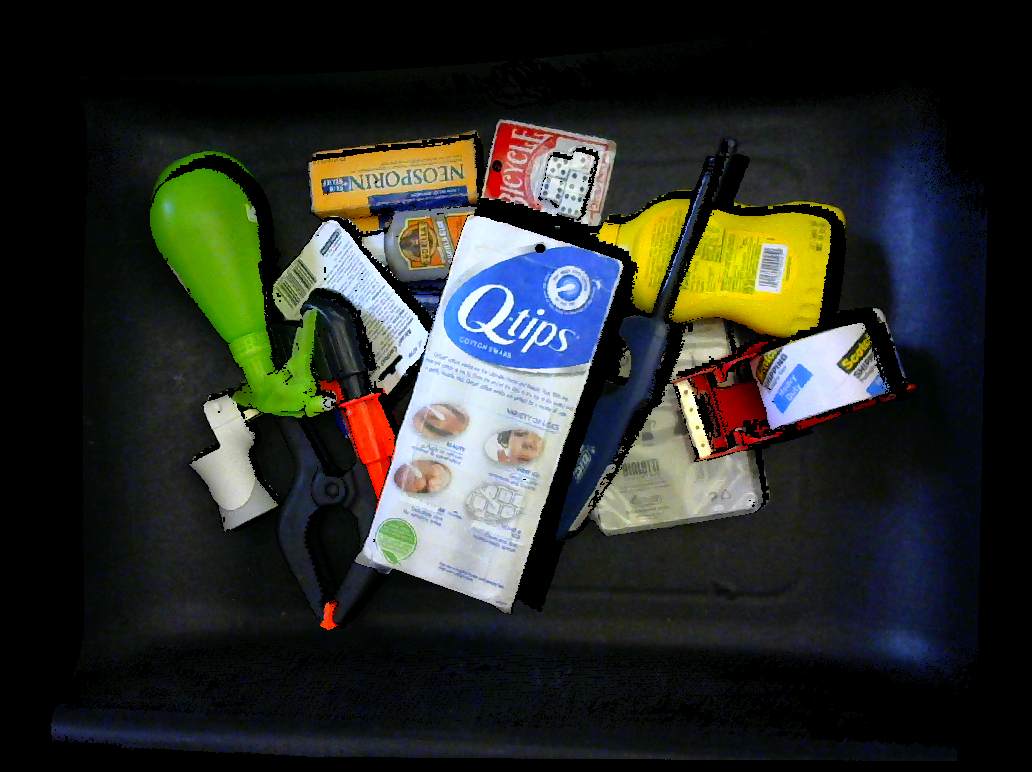}
    \caption{An example of objects from the DexNet 4.0 dataset~\citep{mahler2019learning}.}
    \label{fig:dexnet_input}
\end{figure}

The DexNet dataset of synthetic objects~\citep{mahler2017dex} includes a variety of pick attempts that were not used in training the GQ-CNN model. These samples are divided into a 50\%/50\% train/test split. Each example is labeled as a success if the robot successfully picks and places the object, and a failure otherwise. We ran 100 trials of Algorithm~\ref{alg:general2} with randomized train/test splits, and averaged over the results.

\begin{figure*}[tb]
    \centering 
    \begin{subfigure}[t]{0.475\textwidth}
        \centering
        \includegraphics[width=\textwidth]{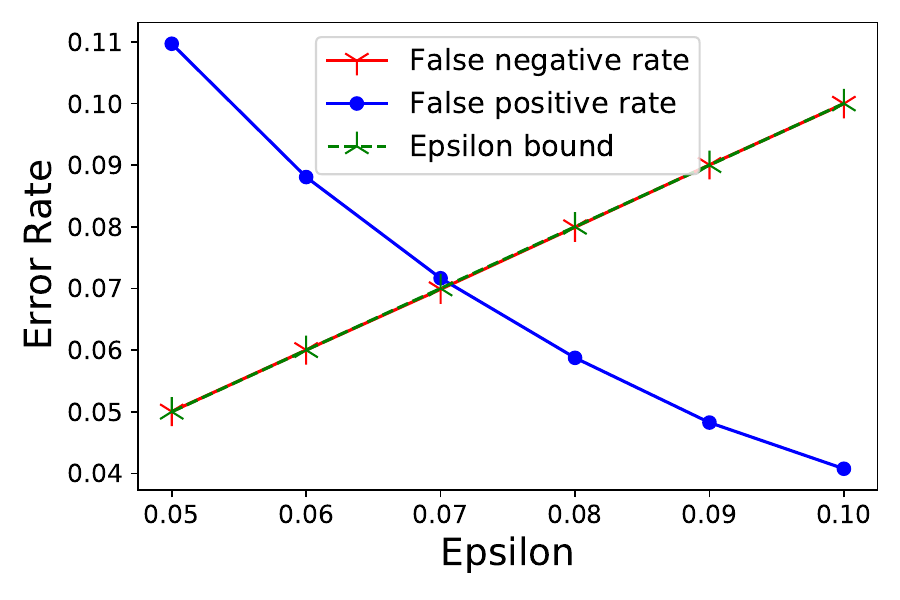}
        \caption{\textbf{Original surrogate safety score.} False negative and false positive rates on the DexNet dataset with varying $\epsilon$, using the predicted probability of a successful pick as the surrogate safety score, $g =  Q_{\theta}(\mathbf{y}, \mathbf{u})$. The theoretical upper bound on epsilon is shown in {\color{OliveGreen} green}. The false negative rate ({\color{red} red}) is within the theoretical bound and the false positive rate ({\color{blue} blue}) is reasonably low.}
        \label{fig:dexnet_noise0}
    \end{subfigure}
    \hfill
    \begin{subfigure}[t]{0.475\textwidth}
        \centering
        \includegraphics[width=\textwidth]{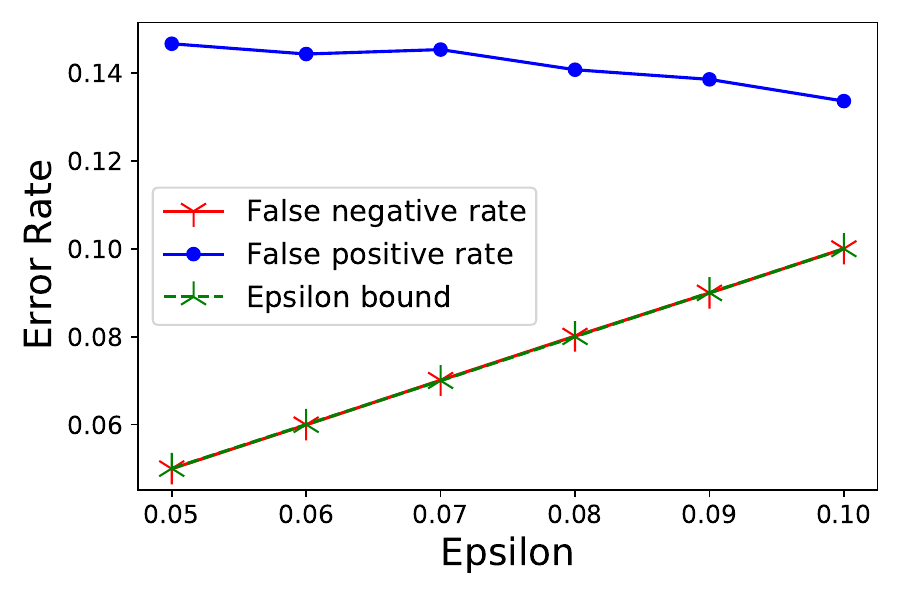}
        \caption{\textbf{Original surrogate safety score $\times$ 0.75 + noise $\times$ 0.25.} False negative and false positive rates on the DexNet dataset with varying $\epsilon$. Here, the surrogate safety score is a weighted sum of the original surrogate safety score (as in Figure~\ref{fig:dexnet_noise0}) and uniformly randomly sampled noise. The theoretical upper bound on epsilon is shown in {\color{OliveGreen} green}. The false negative rate ({\color{red} red}) is within the theoretical bound, but the false positive rate ({\color{blue} blue}) is higher than in Figure~\ref{fig:dexnet_noise0}.}
        \label{fig:dexnet_noise0.25}
    \end{subfigure}
    \vskip\baselineskip
    \begin{subfigure}[t]{0.475\textwidth}
        \centering
        \includegraphics[width=\textwidth]{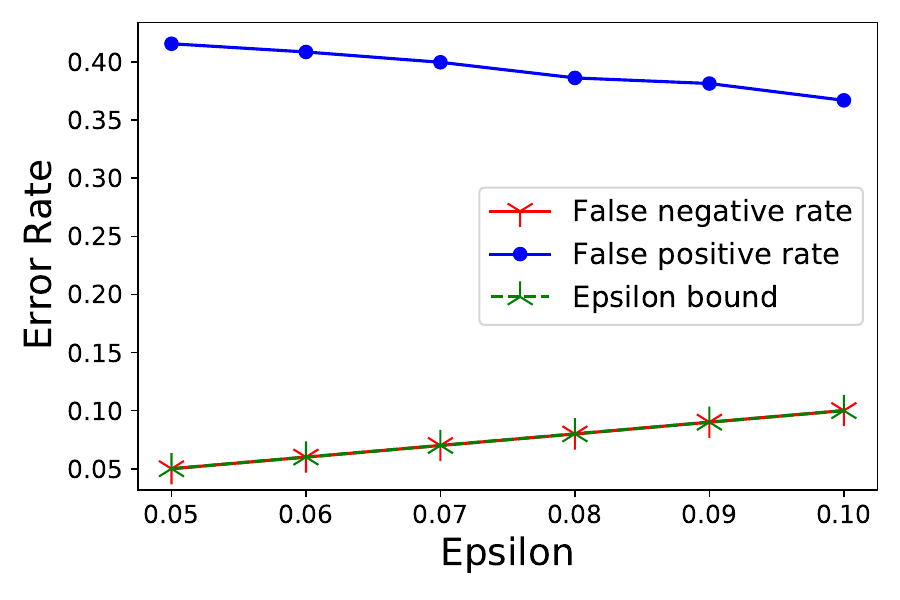}
        \caption{\textbf{Original surrogate safety score $\times$ 0.5 + noise $\times$ 0.5.} False negative and false positive rates on the DexNet dataset with varying $\epsilon$. Here, the surrogate safety score is a weighted sum of the original surrogate safety score (as in Figure~\ref{fig:dexnet_noise0}) and uniformly randomly sampled noise; in this case, there is more noise than in Figure~\ref{fig:dexnet_noise0.25}. The theoretical upper bound on epsilon is shown in {\color{OliveGreen} green}. The false negative rate ({\color{red} red}) is within the theoretical bound, but the false positive rate ({\color{blue} blue}) is higher than in Figures~\ref{fig:dexnet_noise0} and~\ref{fig:dexnet_noise0.25}.}
        \label{fig:dexnet_noise0.5}
    \end{subfigure}
    \hfill
    \begin{subfigure}[t]{0.475\textwidth}
        \centering
        \includegraphics[width=\textwidth]{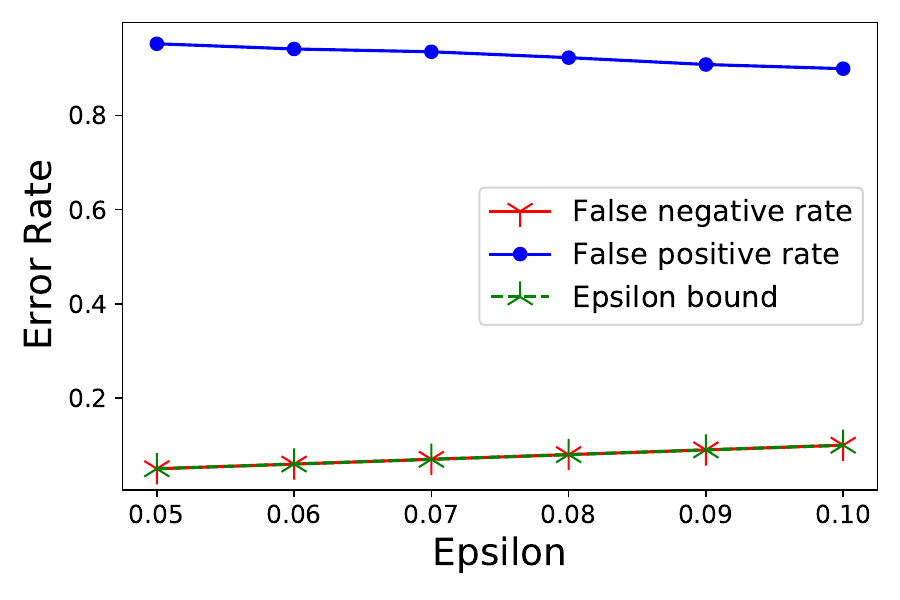}
        \caption{\textbf{Only noise as the surrogate safety score.} False negative and false positive rates on the DexNet dataset with varying $\epsilon$, using uniformly randomly sampled noise as the surrogate safety score, $g = \text{random.random()}$. The theoretical upper bound on epsilon is shown in {\color{OliveGreen} green}. Even in this case, the false negative rate ({\color{red} red}) is within the theoretical bound. However, the false positive rate ({\color{blue} blue}) is very high (with values ranging from 0.95 to 0.90), indicating that the warning system trivially always issues an alert.}
        \label{fig:dexnet_noise1}
    \end{subfigure}
    \caption{False negative rate, false positive rate, and theoretical upper bound on $\epsilon$ for the DexNet dataset of synthetic objects. The false negative rate always remains within the bound, but the false positive rate increases with worse surrogate safety scores.}
\end{figure*}

\subsection{Results and Discussion}
With $\epsilon = 0.05$, we achieved a false negative rate of 0.05, and a false positive rate of 0.11. With $\epsilon = 0.1$, we achieved a false negative rate of 0.10 and a false positive rate of 0.04. The conformal guarantees of our framework hold. A plot of the false negative and false positive rates achieved for different $\epsilon$ values is shown in Figure~\ref{fig:dexnet_noise0}. As before, the false negative rate is within the theoretical bound, and the false positive rate is reasonably low. 

In Figures~\ref{fig:dexnet_noise0.25},~\ref{fig:dexnet_noise0.5}, and~\ref{fig:dexnet_noise1}, we progressively degrade the quality of the surrogate safety score used in Figure~\ref{fig:dexnet_noise0} ($g = Q_{\theta}(\mathbf{y}, \mathbf{u})$) to demonstrate the effects of a worse surrogate safety score or a more inaccurate simulator. 
In fact, in Figure~\ref{fig:dexnet_noise1}, we replace the surrogate safety score $g$ entirely with random noise that is uniformly sampled between 0 and 1, i.e. $g(Y) = \mathcal{U}([0, 1])$. As the plot shows, the false negative rate still remains within the theoretical bound. However, the false positive rate is extremely high (with values ranging from 0.95 to 0.90), indicating that the warning system essentially always issues an alert. This makes sense intuitively, because in this case the surrogate safety score is completely non-informative.
In Figures~\ref{fig:dexnet_noise0.25} and ~\ref{fig:dexnet_noise0.5}, we progressively degrade the quality of the original surrogate safety score (the predicted probability that a pick will be successful) by adding increasing amounts of random noise. (In other words, we progressively decrease the accuracy of the simulator by adding noise.) We take a linear combination of the original surrogate safety score ($g =  Q_{\theta}(\mathbf{y}, \mathbf{u})$) and the uniformly randomly sampled noise, using weights of 0.75 and 0.25 for the original surrogate safety score and the random noise, respectively, in Figure~\ref{fig:dexnet_noise0.25}, and weights of 0.5 and 0.5 in Figure~\ref{fig:dexnet_noise0.5}. Note that the false negative rate always remains within the theoretical bound from Proposition~\ref{prop:safety2}, but the false positive rate increases with increasing noise, indicating that more alerts are being issued and it is more difficult to distinguish between safe and unsafe situations. 

Taken together, these results demonstrate that the guarantees of Algorithm~\ref{alg:general2} hold regardless of the surrogate safety score and simulator model used, and thus, our method can be used to bound the false negative rate of a warning system even if the simulator or prediction model does not have any performance guarantees. However, a surrogate safety score $g$ that is better correlated with the true safety score $f$ and a more accurate simulator model will lead to better empirical performance in terms of the false positive rate, and issue fewer unnecessary alerts. 

\section{Conclusions and Future Work}

In this work, we introduce a broadly applicable framework that uses conformal prediction to tune warning systems for robotics applications. This framework allows us to achieve provable safety assurances with very little data. We demonstrate empirically that the guarantees on the false negative rate hold for a driver alert system and for a robotic grasping system (even with only tens of examples of failure cases), while the false positive rate remains low. 

There are several exciting future directions for this work. One area of particular interest is the application of conformal prediction in non-exchangeable scenarios~\citep{tibshirani2019conformal, Barber2022ConformalPB, Gibbs2022ConformalIF, Cauchois2020RobustVC}, as many robotics settings involve highly correlated time-series data, and robots deployed in the world may encounter distribution shift. 
There have been many recent advances in the conformal prediction literature on relaxing the exchangeability assumption, and leveraging this work could lead to useful developments in robotics. 

Another intriguing extension of this work is exploring conditional safety~\citep{feldman2021improving, Gupta2021OnlineML}, with the goal of providing safety assurances conditioned on specific factors (rather than a marginal guarantee). For example, a driver assistance system that provides a 95\% guarantee on the false negative rate regardless of whether it is raining could be very useful. This system could leverage existing work on conditional coverage guarantees in conformal prediction; however, it would need to be sample-efficient to be useful for robotics settings. 

Two additional interesting and important future directions are studying deployment in industry-scale applications and studying the impact of the predictor on the data that it is trying to predict~\citep{perdomo2020performative} (e.g. examining whether and to what extent the warning system changes behavior or outcomes).

\begin{acks}
The NASA University Leadership Initiative (grant \#80NSSC20M0163) provided funds to assist the authors with their research. This article solely reflects the opinions and conclusions of its authors and not any NASA entity. The authors would like to thank Matteo Zallio for his expertise in crafting Figure~\ref{fig:system}.
\end{acks}

{\small
\bibliographystyle{SageH}
\bibliography{references}
}

\end{document}